\documentclass[11pt]{article}
\usepackage{PRIMEarxiv}
\usepackage{amsfonts} 
\usepackage{listings} 
\usepackage{url} 
\usepackage{cite} 
\usepackage{float}
\usepackage{amsthm,amssymb,amsmath}
\usepackage{xspace}
\usepackage{graphicx} 
\usepackage{subcaption} 

\newtheorem{theorem}{Theorem}[section]

\newtheorem{proposition}[theorem]{Proposition}

\theoremstyle{definition}
\newtheorem{definition}[theorem]{Definition}

\theoremstyle{remark}
\newtheorem{remark}[theorem]{Remark}

\newtheorem{notation}[theorem]{Notation}

\newcommand{\N}{\mathbb{N}}

\newtheorem{corollary}[theorem]{Corollary}
\newtheorem{fact}[theorem]{Fact}
\usepackage{caption}

\DeclareMathOperator{\U}{\mathcal{U}}

\usepackage{tikz-cd}
\usepackage{algorithm}
\usepackage{algorithmic}

\usepackage[T1]{fontenc}
\usepackage{tgbonum}

\begin{document}
\fontfamily{pbk}\selectfont

\title{Latent Space Topology Evolution in Multilayer Perceptrons}
\author{Eduardo Paluzo-Hidalgo\\
Department of Applied Mathematics I,\\ University of Sevilla\\
epaluzo@us.es}

\keywords{Neural Network Interpretability, Topological Data Analysis (TDA), Mapper Algorithm, Latent Representations, Multilayer Perceptrons (MLPs)}

\maketitle

\begin{abstract}
This paper introduces a topological framework for interpreting the internal representations of Multilayer Perceptrons (MLPs). We construct a \emph{simplicial tower}, a sequence of simplicial complexes connected by simplicial maps, that captures how data topology evolves across network layers. Our approach enables bi-persistence analysis: \emph{layer persistence} tracks topological features within each layer across scales, while \emph{MLP persistence} reveals how these features transform through the network. We prove stability theorems for our topological descriptors and establish that linear separability in latent spaces is related to disconnected components in the nerve complexes. To make our framework practical, we develop a combinatorial algorithm for computing MLP persistence and introduce trajectory-based visualisations that track data flow through the network. Experiments on synthetic and real-world medical data demonstrate our method's ability to identify redundant layers, reveal critical topological transitions, and provide interpretable insights into how MLPs progressively organise data for classification.
\end{abstract}

\section{Introduction}
\label{sec:intro}

The widespread deployment of neural networks in critical decision-making systems has created an urgent need for interpretable machine learning models. While these architectures demonstrate remarkable empirical success across diverse domains, their internal mechanisms remain largely opaque, earning them the notorious designation as ``black boxes". This opacity originates from the confluence of several fundamental challenges: the high-dimensional nature of parameter spaces, the compositional complexity introduced by multiple layers of non-linear transformations, and the emergent behaviours that arise from the interplay between architecture and optimisation dynamics.

In this work, we focus on Multilayer Perceptrons (MLPs), the foundational architecture underlying modern deep learning. Despite their apparent simplicity compared to contemporary architectures, MLPs remain ubiquitous as essential components in more complex models. They appear as dense layers in Convolutional Neural Networks (CNNs), as projection heads in Vision Transformers, and as feed-forward networks in Transformer blocks. Understanding the internal representations learned by MLPs thus provides a gateway to interpreting broader classes of neural architectures. Moreover, in safety-critical applications such as medical diagnosis, financial risk assessment, and autonomous systems, the ability to interpret MLP decisions is highly important.

The challenge of neural network interpretability has two complementary research directions. The first employs post-hoc explanation methods that approximate network behaviour through simpler, interpretable models. Techniques such as LIME~\cite{lime} and SHAP~\cite{NIPS2017_7062} construct local linear approximations to explain individual predictions, while attribution methods like Layer-wise Relevance Propagation~\cite{Montavon2019} distribute prediction scores backwards through the network to quantify feature importance. The second direction, which we pursue in this work, seeks to understand the intrinsic geometry and topology of learned representations. This approach recognises that neural networks fundamentally perform a sequence of geometric transformations, progressively reorganising data to achieve linear separability in the final layer.

Recent advances in topological data analysis (TDA) have revealed profound connections between neural network expressivity and the topological properties of data transformations. Work by Saul and Arendt~\cite{saul2018machine} and Pearson~\cite{https://doi.org/10.1155/2013/486363} established the foundational paradigm of using TDA for machine learning explanations, demonstrating how the Mapper algorithm can construct topological representations that reveal both the structure of data and the decision-making processes of trained models. Naitzat et al.~\cite{JMLR:v21:20-345} demonstrated that deep networks systematically reduce the topological complexity of data representations, with each layer acting to simplify the intrinsic structure until classification becomes trivial. This observation was quantified by Suresh et al.~\cite{SURESH2024165}, showing that Betti numbers, which count topological features such as connected components, loops, and voids, decrease with network depth. From a different perspective, Grigsby and Lindsey~\cite{grigsby2021transversalitybenthyperplanearrangements} analysed how ReLU networks partition input space through hyperplane arrangements, revealing geometric constraints on representational capacity. These topological insights have been formalised into bounds by Lee et al.~\cite{lee2023datatopologydependentupperbounds, lee2024definingneuralnetworkarchitecture}, who established relationships between dataset topology and the minimal network width required for successful learning.

Our work advances this topological perspective by introducing a comprehensive framework for tracking and analysing the evolution of data representations across all layers of an MLP. While previous approaches have primarily focused on comparing input and output topologies or analysing individual layers alone, we develop a unified mathematical framework that captures the continuous transformation of topological features throughout the network. Our key innovation lies in constructing a \emph{simplicial tower}, a sequence of simplicial complexes connected by simplicial maps, that encodes the topological structure at each layer while preserving the relationships induced by the network's transformations.

This construction enables several contributions. First, we formalise the relationship between consecutive layer representations through pullback covers and their associated nerve complexes, providing a rigorous foundation for topological analysis of neural networks. Second, we develop a bi-persistence framework that simultaneously tracks two types of topological evolution: \emph{layer persistence}, which captures how topological features evolve as we vary the scale parameter at a fixed layer, and \emph{MLP persistence}, which reveals how features transform as data flows through the network at a fixed scale. This bi-persistence approach, inspired by the Multiscale Mapper framework~\cite{10.5555/2884435.2884506}, provides complementary views of network behaviour that together yield a comprehensive understanding of how MLPs organise data. Third, we establish stability theorems that guarantee our topological descriptors remain robust under perturbations of the covering choices, ensuring that our insights reflect genuine properties of the learned representations rather than artefacts of our analysis. Fourth, we introduce a trajectory-based visualisation that tracks how individual data points and clusters flow through the network's latent spaces, transforming abstract topological concepts into interpretable geometric narratives. Finally, we develop a practical combinatorial algorithm for computing MLP persistence that avoids the computational challenges of explicitly constructing pullback covers, making our framework applicable to real-world neural networks.

Our approach distinguishes itself from existing work in several ways. Unlike the binary analysis of topological simplification in~\cite{JMLR:v21:20-345}, we capture the complete range of topological transformations across all layers. While~\cite{SURESH2024165} tracks global invariants such as Betti numbers, our framework preserves the detailed structure of how specific topological features persist, merge, or disappear. In contrast to theoretical limitations on network expressiveness~\cite{johnson2018deepskinnyneuralnetworks}, we provide constructive tools for understanding existing architectures. Furthermore, our explicit modelling of relationships between layers through simplicial maps offers a different analysis to those based on activation patterns~\cite{lacombe2021topologicaluncertaintymonitoringtrained}. Finally, \cite{saul2018machine} analyses trained models as black boxes; they only need the input data and final prediction probabilities. Their method works with any trained classifier without needing access to internal representations.
However, our approach performs white-box analysis of the internal architecture, specifically requiring access to all intermediate layer representations in MLPs. 

The practical implications of our framework extend beyond theoretical understanding. By revealing how MLPs progressively organise data, our approach can inform architecture design, suggest optimal layer widths, and identify redundant transformations. The trajectory analysis provides intuitive visualisations for understanding classification decisions, potentially enhancing trust in safety-critical applications. Moreover, the mathematical rigour of our topological framework opens new avenues for proving properties about neural network behaviour and establishing formal guarantees about their decision-making processes.

The remainder of this paper is organised as follows. Section~\ref{sec:background} establishes the necessary mathematical foundations from both machine learning and algebraic topology. Section~\ref{sec:methodology_latent} introduces our construction of simplicial complexes from latent representations and proves key results relating topological and geometric properties. Section~\ref{sec:methodology_persitence} develops the simplicial tower framework and our bi-persistence analysis, including stability theorems and computational algorithms. Section~\ref{sec:rel_mapper} formally connects our approach to the Multiscale Mapper framework. Section~\ref{sec:methodology_trajectories} presents our trajectory-based visualisation methodology. Section~\ref{sec:experiments} demonstrates the practical application of our framework through synthetic examples and real-world medical data. Finally, Section~\ref{sec:conclusions} discusses limitations and future research directions.

\section{Background}\label{sec:background}

This section establishes the mathematical foundations for our topological analysis of MLPs. We begin with the machine learning framework (Section~\ref{subsec:ml_framework}), introduce the necessary topological constructions (Section~\ref{subsec:topo_concepts}), present the theory of persistent homology (Section~\ref{subsec:persistent}), and conclude with the Mapper framework that inspires our approach (Section~\ref{subsec:mapper}).

\subsection{Machine Learning Framework}\label{subsec:ml_framework}

We formalise the classification problem and multilayer perceptron architecture that form the basis of our analysis.

\begin{definition}[Classification Problem]
A classification problem consists of a triple $(X, C, \lambda)$ where:
\begin{itemize}
    \item $X \subset \mathbb{R}^d$ is a finite dataset (point cloud),
    \item $C$ is a finite set of class labels,
    \item $\lambda: X \to C$ is the ground truth labeling function.
\end{itemize}
The goal is to learn a function $f: \mathbb{R}^d \to C$ that approximates $\lambda$ on $X$ and generalizes to unseen data.
\end{definition}

While our framework extends to multi-class settings, we focus primarily on binary classification where $C = \{0, 1\}$. The fundamental architecture we analyze is the multilayer perceptron:

\begin{definition}[Multilayer Perceptron]
A multilayer perceptron (MLP) with $m$ hidden layers is a function $F: \mathbb{R}^{n_0} \to \mathbb{R}$ defined as the composition:
\[
F = f_{m+1} \circ f_m \circ \cdots \circ f_1
\]
where each layer function $f_i: \mathbb{R}^{n_{i-1}} \to \mathbb{R}^{n_i}$ is given by:
\[
f_i(x) = \sigma(W_i x + b_i)
\]
with weight matrix $W_i \in \mathbb{R}^{n_i \times n_{i-1}}$, bias vector $b_i \in \mathbb{R}^{n_i}$, and activation function $\sigma: \mathbb{R} \to \mathbb{R}$ applied element-wise. For the output layer, $n_{m+1} = 1$ for binary classification.
\end{definition}

\begin{notation}
Throughout this paper, we adopt the following notation:
\begin{itemize}
    \item $F_i = f_i \circ \cdots \circ f_1: \mathbb{R}^{n_0} \to \mathbb{R}^{n_i}$ denotes the function mapping inputs to the $i$-th layer representation.
    \item $F_{i,j} = f_j \circ f_{j-1} \circ \cdots \circ f_i: \mathbb{R}^{n_{i-1}} \to \mathbb{R}^{n_j}$ denotes the sub-network from layer $i$ to layer $j$.
    \item $X_i = F_i(X) \subset \mathbb{R}^{n_i}$ denotes the image of dataset $X$ at the $i$-th layer.
\end{itemize}
\end{notation}

The ultimate goal of an MLP for binary classification is to transform the input data into a linearly separable configuration:

\begin{definition}[Linear Separability]
Two sets $A, B \subset \mathbb{R}^n$ are \emph{linearly separable} if there exists a hyperplane $H = \{x \in \mathbb{R}^n : w^T x + b = 0\}$ with $w \in \mathbb{R}^n \setminus \{0\}$ and $b \in \mathbb{R}$ such that:
\begin{equation*}
\begin{split}
w^T x + b > 0 &\quad \forall x \in A \\
w^T x + b < 0 &\quad \forall x \in B      
\end{split}
\end{equation*}
\end{definition}

\subsection{Topological Constructions}\label{subsec:topo_concepts}

To analyse the topological evolution of data representations, we employ constructions from algebraic topology. For comprehensive treatments of computational topology, we refer to Dey and Wang \cite{Dey_Wang_2022} and Edelsbrunner and Harer \cite{DBLP:books/daglib/0025666}.

\begin{definition}[Tower]
Let $(A, \leq)$ be a totally ordered set. A \emph{tower} indexed by $A$ is a collection $\mathcal{T} = \{T_a\}_{a \in A}$ of objects in a category, together with morphisms $t_{a,a'}: T_a \to T_{a'}$ for all $a \leq a'$ in $A$, satisfying:
\begin{enumerate}
    \item $t_{a,a} = \text{id}_{T_a}$ for all $a \in A$,
    \item $t_{a',a''} \circ t_{a,a'} = t_{a,a''}$ for all $a \leq a' \leq a''$ in $A$.
\end{enumerate}
We say $\mathcal{T}$ has \emph{resolution} $r$ if $r\le a$ for all $a\in A$.
\end{definition}

Towers can be constructed from various mathematical objects. For our analysis, we focus on towers of topological spaces, covers, and simplicial complexes.

\begin{definition}[Cover and Cover Map]
Let $\mathcal{X}$ be a topological space.
\begin{enumerate}
    \item A \emph{cover} of $\mathcal{X}$ is a collection $\mathcal{U} = \{U_i\}_{i \in I}$ of subsets of $\mathcal{X}$ such that $\bigcup_{i \in I} U_i = \mathcal{X}$.
    \item Given covers $\mathcal{U} = \{U_i\}_{i \in I}$ and $\mathcal{V} = \{V_j\}_{j \in J}$ of $\mathcal{X}$, a \emph{map of covers} is a function $\xi: I \to J$ such that $U_i \subseteq V_{\xi(i)}$ for all $i \in I$.
\end{enumerate}
\end{definition}

\begin{remark}
Throughout this paper, we work with covers consisting of path-connected open sets, ensuring that the nerve theorem (Theorem~\ref{th:nervetheorem}) applies.
\end{remark}

\begin{definition}[Cover Tower]
A \emph{cover tower} of a topological space $X$ is a tower $\mathcal{U} = \{\mathcal{U}_a\}_{a\in A}$ where each $\mathcal{U}_a$ is a cover of $X$, together with maps of covers $u_{a,a'}: \mathcal{U}_a \to \mathcal{U}_{a'}$ for all $a \leq a'$ in $A$.
\end{definition}

Simplicial complexes provide discrete representations of topological spaces that are amenable to computation:

\begin{definition}[Simplex and Simplicial Complex] Simplices and simplicial complexes are defined:
\begin{enumerate}
    \item An \emph{abstract simplex} $\sigma$ is a finite non-empty set. If $|\sigma| = k+1$, we say $\sigma$ is a \emph{$k$-simplex} or has \emph{dimension} $k$.
    \item A \emph{simplicial complex} $K$ is a collection of abstract simplices such that:
    \begin{itemize}
        \item If $\sigma \in K$ and $\tau \subseteq \sigma$ with $\tau \neq \emptyset$, then $\tau \in K$.
        \item If $\sigma, \tau \in K$, then $\sigma \cap \tau \in K$ (including the case $\sigma \cap \tau = \emptyset$).
    \end{itemize}
\end{enumerate}
\end{definition}

\begin{definition}[Simplicial Map]
A \emph{simplicial map} $\varphi: K \to L$ between simplicial complexes is a function $\varphi: V(K) \to V(L)$ on vertices such that for every simplex $\sigma = \{v_0, \ldots, v_k\} \in K$, the image $\{\varphi(v_0), \ldots, \varphi(v_k)\}$ forms a simplex in $L$ (allowing for repeated vertices).
\end{definition}

\begin{definition}[Simplicial Tower]
A \emph{simplicial tower} is a sequence of simplicial complexes $\mathcal{K}_0, \mathcal{K}_1, \ldots, \mathcal{K}_n$ connected by simplicial maps $\varphi_i: \mathcal{K}_{i-1} \rightarrow \mathcal{K}_i$ for $i = 1, 2, \ldots, n$:
\[
\mathcal{K}_0 \xrightarrow{\varphi_1} \mathcal{K}_1 \xrightarrow{\varphi_2} \cdots \xrightarrow{\varphi_n} \mathcal{K}_n
\]
\end{definition}
\begin{definition}[Filtration]
A \emph{filtration} is a simplicial tower where each map is an inclusion, i.e., $\mathcal{K}_0 \subseteq \mathcal{K}_1 \subseteq \cdots \subseteq \mathcal{K}_n$.
\end{definition}

A crucial construction connecting covers to simplicial complexes is the nerve:

\begin{definition}[Nerve of a Cover]
The \emph{nerve} of a cover $\mathcal{U} = \{U_i\}_{i \in I}$ is the simplicial complex $N(\mathcal{U})$ where:
\begin{itemize}
    \item The vertex set is $I$,
    \item A subset $\{i_0, \ldots, i_k\} \subseteq I$ forms a $k$-simplex if and only if $\bigcap_{j=0}^k U_{i_j} \neq \emptyset$.
\end{itemize}
\end{definition}

The nerve theorem establishes when a cover's nerve captures the original space's topology:

\begin{theorem}[Nerve Theorem, Corollary 4G.3 in \cite{hatcher2002algebraic}]\label{th:nervetheorem}
Let $\mathcal{U}$ be a finite open cover of a paracompact space $\mathcal{X}$. If every non-empty intersection of sets in $\mathcal{U}$ is contractible, then $N(\mathcal{U})$ and $\mathcal{X}$ are homotopy equivalent.
\end{theorem}

The following results, established in \cite{Dey_Wang_2022}, relates map of covers and simplicial maps:

\begin{proposition}[Proposition 9.1 in~\cite{Dey_Wang_2022}]\label{prop:induced_simp}
Let $\mathcal{U} = \{U_i\}_{i \in I}$ and $\mathcal{V} = \{V_j\}_{j \in J}$ be two covers of a topological space $\mathcal{X}$. Let $\xi: I \to J$ be a map of covers from $\mathcal{U}$ to $\mathcal{V}$. Then, $\xi$ induces a simplicial map $N(\xi): N(\mathcal{U}) \to N(\mathcal{V})$ between the nerve complexes defined as follows: for each vertex $i \in N(\mathcal{U})$, $N(\xi)(i) = \xi(i) \in N(\mathcal{V})$.
\end{proposition}

\begin{proposition}[Proposition 9.2 in~\cite{Dey_Wang_2022}]\label{prop:same_mapcov}
Let $\mathcal{U} = \{U_i\}_{i \in I}$ and $\mathcal{V} = \{V_j\}_{j \in J}$ be two covers of a topological space $\mathcal{X}$. Let $\xi_1, \xi_2: I \to J$ be two maps of covers from $\mathcal{U}$ to $\mathcal{V}$. Then, the simplicial maps $N(\xi_1), N(\xi_2): N(\mathcal{U}) \to N(\mathcal{V})$ induced by $\xi_1$ and $\xi_2$ respectively are contiguous.
\end{proposition}

\subsection{Persistent Homology and Stability}\label{subsec:persistent}

Persistent homology provides algebraic tools to track topological features across a parameterised family of spaces. We work with homology with coefficients in a field $\mathbb{F}$ (typically $\mathbb{F} = \mathbb{Z}_2$).

\begin{definition}[Persistent Homology]
Given a simplicial tower $\mathcal{K}_0 \xrightarrow{\varphi_1} \mathcal{K}_1 \xrightarrow{\varphi_2} \cdots \xrightarrow{\varphi_n} \mathcal{K}_n$ and a field $\mathbb{F}$, the $p$-th persistent homology is the sequence of homology groups $\{H_p(\mathcal{K}_i,\mathbb{F})\}_{i=0,1,\ldots,n}$ connected by the linear maps $(\varphi_i)_*: H_p(\mathcal{K}_{i-1},\mathbb{F}) \rightarrow H_p(\mathcal{K}_i,\mathbb{F})$ induced by the simplicial maps $\varphi_i$.
\end{definition}

\begin{fact}[Fact 2.11 in~\cite{Dey_Wang_2022}]\label{prop:contiguous_simphom}
Let $K$ and $L$ be simplicial complexes and let $\varphi, \psi: K \to L$ be two contiguous simplicial maps. Then, the induced homomorphisms $\varphi_*, \psi_*: H_p(K) \to H_p(L)$ are identical for all dimensions $p$.
\end{fact}

\begin{definition}[Persistence Diagram]
Given a simplicial tower $\mathcal{K}_0 \xrightarrow{\varphi_1} \mathcal{K}_1 \xrightarrow{\varphi_2} \cdots \xrightarrow{\varphi_n} \mathcal{K}_n$ and its $p$-th persistent homology module, the persistence diagram $\text{Dgm}_p(\mathcal{K})$ is a multiset of points in $\overline{\mathbb{R}}^2$ where each point $(i,j)$ represents a $p$-dimensional homology class born at index $i$ and dying at index $j$.
\end{definition}

\begin{definition}[Bottleneck Distance]
Given two persistence diagrams $D_1$ and $D_2$, the bottleneck distance $d_b(D_1, D_2)$ is:
\[
d_b(D_1, D_2) = \inf_{\gamma} \sup_{p \in D_1} \|p - \gamma(p)\|_{\infty}
\]
where $\gamma$ ranges over all bijections between $D_1$ and $D_2$ (extended with points on the diagonal).
\end{definition}

The stability theorem for persistent homology, established by Cohen-Steiner et al. \cite{cohen2007stability}, guarantees robustness:

\begin{theorem}[Stability of Persistent Homology \cite{cohen2007stability}]\label{th:persistence_stab}
Let $f, g: X \to \mathbb{R}$ be two tame functions. Then:
\[
d_b(\text{Dgm}(f), \text{Dgm}(g)) \leq \|f - g\|_{\infty}
\]
\end{theorem}

For simplicial towers, stability is measured through interleaving:

\begin{definition}[Interleaving of Cover Towers]
Let $\mathcal{U} = \{\mathcal{U}_a\}$ and $\mathcal{V} = \{\mathcal{V}_a\}$ be two cover towers of a topological space $X$ with $\text{res}(\mathcal{U}) = \text{res}(\mathcal{V}) = r$. Given $\eta \geq 0$, we say that $\mathcal{U}$ and $\mathcal{V}$ are $\eta$-interleaved if one can find cover maps $\zeta_a : \mathcal{U}_a \to \mathcal{V}_{a+\eta}$ and $\xi_{a'} : \mathcal{V}_{a'} \to \mathcal{U}_{a'+\eta}$ for all $a, a' \geq r$.
\end{definition}

Stability for towers is established in the following results:

\begin{proposition}[Proposition 9.17 (ii) in~\cite{Dey_Wang_2022}]\label{prop:pullback_interleaved}
Let $f:X\rightarrow Z$ be a continuous function and $\mathcal{U}$ and $\mathcal{V}$ be two $\eta$-interleaved tower of covers of $Z$. Then, the pullback cover towers are also $\eta$-interleaved.
\end{proposition}

\begin{proposition}[Proposition 9.18 in~\cite{Dey_Wang_2022}]\label{prop:nerve_interleaved}
Let $\mathcal{U}$ and $\mathcal{V}$ be two $\eta$-interleaved cover towers of $X$ with $\text{res}(\mathcal{U}) = \text{res}(\mathcal{V})$. Then, $N(\mathcal{U})$ and $N(\mathcal{V})$ are also $\eta$-interleaved.
\end{proposition}

\begin{theorem}[Theorem 4.3 in~\cite{Dey_Wang_2022}]\label{th:simptower_stab}
Let $K$ and $L$ be two simplicial towers and $V_K$ and $V_L$ their homology towers, respectively, that are tame. Then:
\[
d_b(\text{Dgm}(V_K), \text{Dgm}(V_L)) \leq d_I(K, L)
\]
\end{theorem}

\subsection{Mapper and Multiscale Mapper}\label{subsec:mapper}

The Mapper algorithm, introduced by Singh et al. \cite{Singh2007}, provides a framework for visualising high-dimensional data through topological lenses and has been widely applied in different contexts~\cite{madukpe2025comprehensivereviewmapperalgorithm}.

\begin{definition}[Mapper]
Given a topological space $\mathcal{X}$, a continuous function $f: \mathcal{X} \rightarrow \mathcal{Z}$ to a parameter space $\mathcal{Z}$, and a finite cover $\mathcal{U} = \{U_\alpha\}_{\alpha \in A}$ of $\mathcal{Z}$, the Mapper construction consists of:
\begin{enumerate}
    \item Taking the pullback cover $f^*(\mathcal{U})$ of $\mathcal{X}$, where $f^*(\mathcal{U}) = \{f^{-1}(U_\alpha)\}_{\alpha \in A}$.
    \item Computing the nerve of the refined pullback cover, where each $f^{-1}(U_\alpha)$ is replaced by its connected components:
    $$\mathcal{U}_f = \{\text{Connected components of }f^{-1}(U_\alpha) \mid \alpha \in A\}$$
    \item The Mapper of $(\mathcal{X}, f, \mathcal{U})$ is defined as:
    $$\mathrm{M}(\mathcal{X}, f, \mathcal{U}) = \mathrm{N}(\mathcal{U}_f)$$
\end{enumerate}
\end{definition}

To address sensitivity to cover choice, Dey et al. \cite{10.5555/2884435.2884506} introduced Multiscale Mapper:

\begin{definition}[Multiscale Mapper]
Given a continuous function $f: \mathcal{X} \rightarrow \mathcal{Z}$ and a cover tower $\mathcal{U} = \{\mathcal{U}_\varepsilon\}_{\varepsilon \in A}$ of $\mathcal{Z}$, the Multiscale Mapper is the tower of simplicial complexes:
$$\mathrm{MM}(\mathcal{U}, f) = \{N(f^*(\mathcal{U}_\varepsilon))\}_{\varepsilon \in A}$$
connected by the simplicial maps induced by the maps of covers of the pullback cover tower $f^*(\mathcal{U})$.
\end{definition}

This construction provides the conceptual foundation for our analysis of neural networks, where we view each layer transformation as inducing a Mapper-type construction.

\begin{proposition}[Proposition 9.13 in~\cite{Dey_Wang_2022}]\label{prop:induce_mapcov}
Let $f: X\rightarrow Z$, and $\mathcal{U}$ and $\mathcal{V}$ be two covers of $Z$ with a map of covers $\xi:\mathcal{U}\rightarrow \mathcal{V}$. Then, there is a corresponding map of covers between their respective pullback covers.
\end{proposition}

\section{Proposed methodology}\label{sec:methodology}
In the following, we consider an MLP $F$ defined as the composition of $m$ hidden layers $\{\mathbb{R}^{n_{i-1}}\xrightarrow{f_i}\mathbb{R}^{n_i}\}_{i=1}^{m+1}$ with $\mathbb{R}^{n_{m+1}}=\mathbb{R}$. Then, $F_i: \mathbb{R}^{n_0} \to \mathbb{R}^{n_i}$ denotes the function mapping inputs to their representations at the $i$-th layer and $X_i = F_i(X)$ represents the image of the dataset at the $i$-th layer. Besides, $\mathcal{U}_i = \{U_j\}_{j \in J_i}$ represents a cover of $X_i$ with index set $J_i$. 

The organisation is as follows: in Section~\ref{sec:methodology_latent} we describe and provide results about the nerve complex of the latent representation of a given layer. Then, in Section~\ref{sec:methodology_persitence}, we describe how to relate the topological structure of each layer so that we can obtain simplicial towers and then compute persistent homology. In Section~\ref{sec:rel_mapper}, we relate our results to the Multiscale mapper. Finally, in Section~\ref{sec:methodology_trajectories}, we describe how we can track specific trajectories to visualise and provide interpretability to the decision-making of MLPs.

\subsection{Simplicial-based latent representation}\label{sec:methodology_latent}

To analyse the topological structure of neural network representations, we must transform the latent space at each layer into a topological object that captures meaningful structure. Given a dataset $X$ and its representation $X_i = F_i(X)$ at the $i$-th layer, we view $X_i$ as a subspace of $\mathbb{R}^{n_i}$ with the subspace topology inherited from the standard topology on $\mathbb{R}^{n_i}$. Let us consider a cover and compute its nerve complex $\mathcal{K}_i:= N(\mathcal{U}_i)$. The Nerve theorem (Theorem~\ref{th:nervetheorem}) states that under some conditions, the nerve complex is homotopy equivalent to $X_i$, and hence we can expect to recover information about the space we are interested in. For classification problems, we are particularly interested in how different classes are separated in the latent space, and under the right choice of the cover, we can keep linear separability on the nerve. To do so, we first recall that two sets $A$ and $B$ in a topological space are Hausdorff-separated if disjoint open sets $U_1$ and $U_2$ exist such that $A\subset U_1$ and $B\subset U_2$. Then, we have the following result regarding linear separability:

\begin{proposition}\label{proposition:separation}
Given two disjoint sets $A, B \subset \mathbb{R}^n$:
\begin{enumerate}
    \item If $A$ and $B$ are linearly separable in $\mathbb{R}^n$, then they are Hausdorff-separated in the subspace topology of $\mathbb{R}^n$.
    \item If $A$ and $B$ are Hausdorff-separated, then there exists a cover $\mathcal{U}$ of $X = A \cup B$ such that the nerve complex $\mathcal{K} = N(\mathcal{U})$ has two disconnected components $\mathcal{K}_A$ and $\mathcal{K}_B$ corresponding to $A$ and $B$ respectively.
\end{enumerate}
\end{proposition}

\begin{proof}
We prove each statement separately.

\textbf{(1)} Suppose $A$ and $B$ are linearly separable. Then there exist $w \in \mathbb{R}^n \setminus \{0\}$ and $b \in \mathbb{R}$ such that the hyperplane $H = \{x \in \mathbb{R}^n : w^T x + b = 0\}$ satisfies: (1) $w^T x + b > 0 \quad \forall x \in A$ and (2) $ w^T x + b < 0 \quad \forall x \in B$. Then, define the continuous function $h: \mathbb{R}^n \to \mathbb{R}$ by $h(x) = w^T x + b$. Since $A$ and $B$ are disjoint and finite (as subsets of a point cloud), they are compact. Therefore: (1) $\delta_A := \min_{x \in A} h(x) > 0$ (by continuity of $h$ on compact $A$), and (2) $\delta_B := \max_{x \in B} h(x) < 0$ (by continuity of $h$ on compact $B$).

Define the open sets: $U_A := h^{-1}((0, \infty)) = \{x \in \mathbb{R}^n : h(x) > 0\}$ and $ U_B := h^{-1}((-\infty, 0)) = \{x \in \mathbb{R}^n : h(x) < 0\}$. These sets are open as preimages of open intervals under the continuous function $h$. By construction $A \subset U_A$, $B \subset U_B$, and $U_A \cap U_B = \emptyset$. Therefore, $A$ and $B$ are Hausdorff-separated.

\textbf{(2)} Suppose $A$ and $B$ are Hausdorff-separated by disjoint open sets $U_A \supset A$ and $U_B \supset B$. We construct an explicit cover that yields the desired nerve complex. Since $A$ and $B$ are finite sets (as subsets of a point cloud $X$), we can choose for each point an open neighbourhood contained entirely within the appropriate separating set: (1) for each $a \in A$, since $a \in U_A$ and $U_A$ is open, there exists $\varepsilon_a > 0$ such that the open ball $B_{\varepsilon_a}(a) \subset U_A$; and, for each $b \in B$, since $b \in U_B$ and $U_B$ is open, there exists $\varepsilon_b > 0$ such that the open ball $B_{\varepsilon_b}(b) \subset U_B$. Now, define the cover:
\[
\mathcal{U} := \{B_{\varepsilon_a}(a) : a \in A\} \cup \{B_{\varepsilon_b}(b) : b \in B\}
\]

Now, the nerve $\mathcal{K} = N(\mathcal{U})$ is the simplicial complex with vertex set $I = A \cup B$ and higher dimensional simplices $\{x_0, x_1, \ldots, x_k\} \subseteq I$ forms a $k$-simplex in $\mathcal{K}$ if and only if $\bigcap_{i=0}^k B_{\varepsilon_{x_i}}(x_i) \neq \emptyset$. Finally, no simplex in $\mathcal{K}$ can contain vertices from both $A$ and $B$. Since there are no edges between vertices in $\mathcal{K}_A$ and vertices in $\mathcal{K}_B$, the nerve complex $\mathcal{K}$ has exactly two connected components corresponding to $A$ and $B$.
\end{proof}

\subsection{Persistence on the latent representations}\label{sec:methodology_persitence}
In the following, we develop a framework to track the topological evolution of data representations across network layers by computing the persistent homology of simplicial towers associated with the MLP. 

\subsubsection{From latent representations to a sequence of covers}

To study how data representations evolve across network layers, we need to establish relationships between the topological structures at each layer. Let us define a relation between the index sets of consecutive covers:
\begin{equation}\label{eq:relation}
R_i \subseteq J_{i-1} \times J_{i} \quad \text{such that} \quad \alpha \, R_i \, \beta \iff f_i(U_\alpha) \cap U_\beta \neq \emptyset,   
\end{equation}
where $U_\alpha \in \mathcal{U}_{i-1}$ and $U_\beta \in \mathcal{U}_{i}$. This relation captures all possible connections between open sets in consecutive layers based on the neural network function $f_i$. As we have covers of different spaces, which are the latent representations, we need to extend the definition of the map of covers.

\begin{definition}[Map of Covers through a continuous function]
Given two covers $\mathcal{U}=\{U_i\}_{i\in I}$ of a space $\mathcal{X}$ and $\mathcal{V}=\{V_j\}_{j\in J}$ of a space $\mathcal{Y}$, respectively, and consider a continuous function $f:\mathcal{X}\rightarrow \mathcal{Y}$. Then, a map of covers (through $f$) from $\mathcal{U}$ to $\mathcal{V}$ is a set map $\xi:I\rightarrow J$ such that $f(U_i)\subseteq V_{\xi(i)}$ for every $i\in I$.
\end{definition}

\begin{remark}\label{rm:extension}
Propositions~\ref{prop:induced_simp}, \ref{prop:same_mapcov}, and \ref{prop:induce_mapcov} also hold for maps of covers through continuous functions.
\end{remark} 

Note that the relation in Equation~\ref{eq:relation} does not directly define a map of covers, since a single index $\alpha \in J_{i-1}$ may be related to multiple indices in $J_i$. Also, let us remark that cover towers are supposed to be a sequence of covers of the same space. However, we are interested in the evolution of the latent representations and then define cover towers for the specific case of MLPs. 

\begin{definition}[Cover tower of an MLP]
Given an MLP $F$ defined as the composition of $m$ hidden layers $\{\mathbb{R}^{n_{i-1}}\xrightarrow{f_i}\mathbb{R}^{n_i}\}_{i=1}^{m+1}$ with $n_{m+1}=1$. A cover tower of $F$ is a sequence of covers $\{\mathcal{U}_{i-1}\xrightarrow{\xi_i}\mathcal{U}_i\}_{i=1}^{m+1}$ where:
\begin{itemize}
    \item Each $\mathcal{U}_i = \{U_{\alpha}\}_{\alpha \in J_i}$ is a cover of the image space $X_i = F_i(X) \subset \mathbb{R}^{n_i}$ of the $i$-th layer
    \item Each $\xi_i: J_{i-1} \to J_i$ is a map of covers through $f_i$.
\end{itemize}
\end{definition}

In the following, we will call indistinctly cover towers and cover towers for MLPs as cover towers.

\begin{corollary}[Compatibility of composed cover maps]
Let $\{\mathcal{U}_{i-1}\xrightarrow{\xi_i}\mathcal{U}_i\}_{i=1}^{m+1}$ be a cover tower of an MLP $F$. For any $1 \leq i < j \leq m+1$, the composition $\xi_{i,j} = \xi_j \circ \xi_{j-1} \circ \cdots \circ \xi_{i+1} \circ \xi_i$ is a map of covers through the composed function $f_{i,j} = f_j \circ f_{j-1} \circ \cdots \circ f_{i+1} \circ f_i$.
\end{corollary}

\begin{proof}
We proceed by induction on the length of the composition. The base case is trivial since $\xi_i$ is already a map of covers through $f_i$ by definition.

For the inductive step, assume that $\xi_{i,k} = \xi_k \circ \cdots \circ \xi_i$ is a map of covers through $f_{i,k} = f_k \circ \cdots \circ f_i$ for some $k$ where $i < k < j$. We need to show that $\xi_{i,k+1} = \xi_{k+1} \circ \xi_{i,k}$ is a map of covers through $f_{i,k+1} = f_{k+1} \circ f_{i,k}$.

Let $\alpha \in J_{i-1}$ and consider $U_\alpha \in \mathcal{U}_{i-1}$. By the induction hypothesis, we have:
\[f_{i,k}(U_\alpha) \subseteq U_{\xi_{i,k}(\alpha)}\]

Since $\xi_{k+1}$ is a map of covers through $f_{k+1}$, we have:
\[f_{k+1}(U_{\xi_{i,k}(\alpha)}) \subseteq U_{\xi_{k+1}(\xi_{i,k}(\alpha))} = U_{\xi_{i,k+1}(\alpha)}\]

Therefore:
\[f_{i,k+1}(U_\alpha) = f_{k+1}(f_{i,k}(U_\alpha)) \subseteq f_{k+1}(U_{\xi_{i,k}(\alpha)}) \subseteq U_{\xi_{i,k+1}(\alpha)}\]

This shows that $\xi_{i,k+1}$ is a map of covers through $f_{i,k+1}$, completing the induction step. We conclude that $\xi_{i,j}$ is a map of covers through $f_{i,j}$ for all $1 \leq i < j \leq m+1$.
\end{proof}

To guarantee the existence of the map of covers between the covers of the different latent spaces, we will induce the cover tower through a backpropagation process following the same construction used in Mapper. Given a cover $\U_i$ of the $i$-th latent space $X_i$, we can compute the connected components of the preimage of each open set in $\U_i$, which we will call the pullback $f^*_i$. Then, we define the pullback of the cover in the following way:
\begin{equation}\label{eq:pullback_cover}
 \U_{i-1}:= f_i^*(\U_i) = \{\text{Connected components of }f_i^{-1}(U_\alpha)\, \mid \, U_\alpha\in \U_i\}.   
\end{equation}

Remark that $\U_{i-1}$ is a cover of $X_{i-1}$. To define the map of covers $\xi_i:J_{i-1}\rightarrow J_i$, we need to map each open set in $\U_{i-1}$ with its associated open set in $\U_i$. This reduces the number of choices as we only need to fix the cover on the output layer $\U_{m+1}$. Proposition~\ref{prop:same_mapcov} ensures that any maps of covers defined between the same two covers induce contiguous simplicial maps. Consequently, these maps induce identical homomorphisms at the homology level (Fact~\ref{prop:contiguous_simphom}). Besides, we only consider this methodology assuming that the preimage has finitely many connected components.

\begin{corollary}
Given an MLP $F$, the sequence of covers $\{\mathcal{U}_{i-1}\xrightarrow{\xi_i}\mathcal{U}_i\}$ where each $\U_i$ is a pullback cover as defined in Equation~\ref{eq:pullback_cover} with the induced map of cover is a cover tower of $F$.    
\end{corollary}

Let us remark that given an MLP $F$, we have a cover tower:
$$\U_0\xrightarrow{f_1}\U_1\xrightarrow{f_2}\cdots \xrightarrow{f_m}\U_m\xrightarrow{f_{m+1}}\U_{m+1}$$
and it satisfies that $(f_i\circ f_{i+1}\circ \cdots \circ f_{m+1})^*=f_i^*\circ f_{i+1}^*\circ \cdots \circ f_{m+1}^*$ for all $i\in\{1,\dots,m+1\}$. 

\subsubsection{From sequence of covers to simplicial towers and persistent homology}\label{sec:pullback_cover_tower}

For each cover $\mathcal{U}_i$, we then compute its nerve complex $\mathcal{K}_i := N(\mathcal{U}_i)$. The map of covers $\xi_i: J_{i-1} \to J_i$ (through $f_i$) induces a simplicial map $\varphi_i: \mathcal{K}_{i-1} \to \mathcal{K}_i$ (Proposition~\ref{prop:induced_simp} and Remark~\ref{rm:extension}). Then, we can construct a simplicial tower:
$$
\mathcal{K}_0 \xrightarrow{\varphi_1} \mathcal{K}_1 \xrightarrow{\varphi_2} \cdots \xrightarrow{\varphi_m} \mathcal{K}_m \xrightarrow{\varphi_{m+1}} \mathcal{K}_{m+1}
$$
Applying the homology functor with coefficients in a field $\mathbb{F}$, we obtain a sequence of vector spaces and linear maps:

$$
H_p(\mathcal{K}_0) \xrightarrow{(\varphi_1)_*} H_p(\mathcal{K}_1) \xrightarrow{(\varphi_2)_*} \cdots \xrightarrow{(\varphi_m)_*} H_p(\mathcal{K}_m) \xrightarrow{(\varphi_{m+1})_*} H_p(\mathcal{K}_{m+1})
$$
This sequence defines the $p$-th persistent homology module of the simplicial tower, capturing how topological features evolve across the neural network layers. The resulting persistence diagrams reveal the evolution of the topology of the latent space through the MLP, showing the inner processing of the MLP towards a linearly separable output. The persistent homology of a simplicial tower can be computed using algorithms described in~\cite{Dey_Wang_2022} and in~\cite{Kerber2018}\footnote{See \url{https://bitbucket.org/schreiberh/sophia/src/master/} for their specific code to compute persistent homology of simplicial towers.}.

It is hard to control whether two covers are homeomorphic. However, we can explicitly relate the homology groups of two simplicial towers so that we can compute the existing differences when different covers are chosen.

We can iterate the result in Proposition~\ref{prop:induce_mapcov} to build the following diagrams.

\begin{corollary}[Cover towers related under a map of covers]
Given two cover towers $\mathcal{U} = \{\mathcal{U}_{i}\}_{i=0}^{m+1}$ and $\mathcal{V} = \{\mathcal{V}_{i}\}_{i=0}^{m+1}$ induced by pullback operations from the last covers $\mathcal{U}_{m+1}$ and $\mathcal{V}_{m+1}$ with a map of covers $\rho_{m+1}: \mathcal{U}_{m+1} \rightarrow \mathcal{V}_{m+1}$, there exist maps of covers $\rho_i: \mathcal{U}_i \rightarrow \mathcal{V}_i$ for each $i \in \{0,1,\ldots,m\}$ that form the following diagram:

\[
\begin{tikzcd}
\mathcal{U}_0 \arrow[r, "\xi_1"] \arrow[d, "\rho_0"'] & \mathcal{U}_{1} \arrow[r, "\xi_2"] \arrow[d, "\rho_1"'] &\cdots  \arrow[r, "\xi_m"] &\mathcal{U}_{m} \arrow[r, "\xi_{m+1}"] \arrow[d, "\rho_m"'] & \mathcal{U}_{m+1}\arrow[d, "\rho_{m+1}"']\\
\mathcal{V}_0 \arrow[r, "\eta_1"]  & \mathcal{V}_{1} \arrow[r, "\eta_2"]  &\cdots  \arrow[r, "\eta_m"] &\mathcal{V}_{m} \arrow[r, "\eta_{m+1}"] & \mathcal{V}_{m+1}
\end{tikzcd}
\]
Furthermore, it induces the following diagram:
\[
\begin{tikzcd}
K_0 \arrow[r, "\varphi_1"] \arrow[d, "\phi_0"'] & K_{1} \arrow[r, "\varphi_2"] \arrow[d, "\phi_1"'] &\cdots  \arrow[r, "\varphi_m"] &K_{m} \arrow[r, "\varphi_{m+1}"] \arrow[d, "\phi_m"'] & K_{m+1}\arrow[d, "\phi_{m+1}"']\\
L_0 \arrow[r, "\psi_1"]  & L_{1} \arrow[r, "\psi_2"]  &\cdots  \arrow[r, "\psi_m"] &L_{m} \arrow[r, "\psi_{m+1}"] & L_{m+1}
\end{tikzcd}
\]
where $K_i = N(\mathcal{U}_i)$, $L_i = N(\mathcal{V}_i)$, $\varphi_i = N(\xi_i)$, $\psi_i = N(\eta_i)$, and $\phi_i = N(\rho_i)$.
\end{corollary}
\begin{proof}
We begin with the given map of covers $\rho_{m+1}: \mathcal{U}_{m+1} \rightarrow \mathcal{V}_{m+1}$ at the last layer. Our goal is to construct maps of covers $\rho_i: \mathcal{U}_i \rightarrow \mathcal{V}_i$ for all $i \in \{0,1,\ldots,m\}$. Recall that each cover $\mathcal{U}_i$ and $\mathcal{V}_i$ is constructed by a pullback operation from the subsequent layer. Specifically, for any $i \in \{0,1,\ldots,m\}$:
$$\mathcal{U}_i = \{\text{Connected components of } f_{i+1}^{-1}(U_\alpha) \mid U_\alpha \in \mathcal{U}_{i+1}\}$$
$$\mathcal{V}_i = \{\text{Connected components of } f_{i+1}^{-1}(V_\beta) \mid V_\beta \in \mathcal{V}_{i+1}\}$$
where $f_{i+1}$ is the function representing the $(i+1)$-th layer of the MLP.
By Proposition~\ref{prop:induce_mapcov}, for each $i \in \{0,1,\ldots,m\}$, the function $f_{i+1}: X_i \rightarrow X_{i+1}$ and the map of covers $\rho_{i+1}: \mathcal{U}_{i+1} \rightarrow \mathcal{V}_{i+1}$ induce a map of covers $\rho_i: \mathcal{U}_i \rightarrow \mathcal{V}_i$. Specifically, for each open set $U \in \mathcal{U}_i$, which is a connected component of $f_{i+1}^{-1}(U')$ for some $U' \in \mathcal{U}_{i+1}$, we define $\rho_i(U)$ to be the connected component of $f_{i+1}^{-1}(\rho_{i+1}(U'))$ that contains $U$, establishing the first diagram.

For the second, we apply the nerve $N$ to the first diagram. By Proposition~\ref{prop:induced_simp}, a map of covers induces a simplicial map between the corresponding nerve complexes. Therefore, for each $i \in \{0,1,\ldots,m+1\}$: (1) $\xi_i: \mathcal{U}_{i-1} \rightarrow \mathcal{U}_i$ induces $\varphi_i = N(\xi_i): K_{i-1} \rightarrow K_i$; (2) $\eta_i: \mathcal{V}_{i-1} \rightarrow \mathcal{V}_i$ induces $\psi_i = N(\eta_i): L_{i-1} \rightarrow L_i$; and, (3) $\rho_i: \mathcal{U}_i \rightarrow \mathcal{V}_i$ induces $\phi_i = N(\rho_i): K_i \rightarrow L_i$. This establishes the second diagram of simplicial maps between the corresponding nerve complexes.
\end{proof}

\begin{corollary}
The following diagram commutes:
\[
\begin{tikzcd}
H_p(K_0) \arrow[r, "\varphi_1"] \arrow[d, "(\phi_0)^*"'] & H_p(K_{1}) \arrow[r, "(\varphi_2)^*"] \arrow[d, "(\phi_1)^*"'] &\cdots  \arrow[r, "(\varphi_m)^*"] &H_p(K_{m}) \arrow[r, "\varphi_{m+1}"] \arrow[d, "(\phi_m)^*"'] & H_p(K_{m+1})\arrow[d, "(\phi_{m+1})^*"']\\
H_p(L_0) \arrow[r, "(\psi_1)^*"]  & H_p(L_{1}) \arrow[r, "(\psi_2)^*"]  &\cdots  \arrow[r, "(\psi_m)^*"] &H_p(L_{m}) \arrow[r, "(\psi_{m+1})"] & H_p(L_{m+1})
\end{tikzcd}
\]    
\end{corollary}
\begin{proof}
The composition of map of covers is a map of covers and, hence, the two possible paths for each square are map of covers between the same covers and then the induced simplicial maps are contiguous (see Proposition~\ref{prop:same_mapcov}), defining exactly the same maps at homology level (see Fact~\ref{prop:contiguous_simphom}), so the homology tower diagram does commute.     
\end{proof}

We can extend the tower for more than two covers and reach the following stability result.

\begin{theorem}[Stability of Pullback Cover Towers]\label{th:stab_pullback}
Let $F$ be an MLP with layers $\{\mathbb{R}^{n_{i-1}}\xrightarrow{f_i}\mathbb{R}^{n_i}\}_{i=1}^{m+1}$ where $\mathbb{R}^{n_{m+1}}=\mathbb{R}$, and let $X_i = F_i(X)$ be the image of dataset $X$ at the $i$-th layer. Consider two cover towers $\mathcal{U}_{m+1}$ and $\mathcal{V}_{m+1}$ of $X_{m+1}$ that are $\eta$-interleaved with the same resolution. Let $\mathcal{U}_{i}$ and $\mathcal{V}_{i}$ be the pullback cover towers defined recursively for each $i \in \{0,1,\ldots,m\}$. Then:

\begin{enumerate}
    \item For each $i \in \{0,1,\ldots,m\}$, the pullback cover towers $\mathcal{U}_{i}$ and $\mathcal{V}_{i}$ are $\eta$-interleaved.

    \item If $N(\mathcal{U}_{i})$ and $N(\mathcal{V}_{i})$ denote the nerve complexes of these covers, then for each $i \in \{0,1,\ldots,m+1\}$, the simplicial towers $N(\mathcal{U}_{i})$ and $\N(\mathcal{V}_{i})$ are also $\eta$-interleaved.

    \item The bottleneck distance between the persistence diagrams of these simplicial towers is bounded by $\eta$ for any homology dimension $p$:
    $$d_b(\mathrm{Dgm}_p(N(\mathcal{U}_{i})), \mathrm{Dgm}_p(N(\mathcal{V}_{i}))) \leq \eta$$
\end{enumerate}
\end{theorem}

\begin{proof}
We prove each claim in order:

\begin{enumerate}
    \item We proceed by backward induction on $i$, starting from $i = m+1$. By assumption, the cover towers $\mathcal{U}_{m+1}$ and $\mathcal{V}_{m+1}$ are $\eta$-interleaved. For the inductive step, assume that $\mathcal{U}_{i}$ and $\mathcal{V}_{i}$ are $\eta$-interleaved for some $i \in \{1,\ldots,m\}$. Then, by Proposition~\ref{prop:pullback_interleaved}, given the continuous function $f_i: X_{i-1} \rightarrow X_i$ and the two $\eta$-interleaved cover towers of $X_i$, their pullback cover towers via $f_i$ are also $\eta$-interleaved. Applying this result, we conclude that the pullback cover towers $\mathcal{U}_{i-1}$ and $\mathcal{V}_{i-1}$ are also $\eta$-interleaved. By induction, we establish that $\mathcal{U}_{i}$ and $\mathcal{V}_{i}$ are $\eta$-interleaved for all $i \in \{0,1,\ldots,m+1\}$.

    \item For each $i \in \{0,1,\ldots,m+1\}$, we have established that the cover towers $\{\mathcal{U}_{i}\}$ and $\{\mathcal{V}_{i}\}$ are $\eta$-interleaved. By Proposition~\ref{prop:nerve_interleaved}, if two cover towers with the same resolution are $\eta$-interleaved, then their nerve complexes are also $\eta$-interleaved. Therefore, the simplicial towers $N(\mathcal{U}_{i})$ and $N(\mathcal{V}_{i})$ are $\eta$-interleaved for each $i \in \{0,1,\ldots,m+1\}$.

    \item By Theorem~\ref{th:simptower_stab}, for tame $\eta$-interleaved simplicial towers $N(\mathcal{U}_{i})$ and $N(\mathcal{V}_{i})$ we have:
    $$d_b(\mathrm{Dgm}_p(N(\mathcal{U}_{i})), \mathrm{Dgm}_p(N(\mathcal{V}_{i}))) \leq \eta$$
    for any homology dimension $p$.
\end{enumerate}
\end{proof}

\subsubsection{Pullback of Čech filtrations}

In the following, we focus on when we are given a point cloud dataset. We can consider the Čech complex (or Vietoris-Rips complexes), which is the nerve of 
\begin{equation}
\U_{m+1}^\varepsilon:=\left\{B\left(x,\frac{\varepsilon}{2}\right)\,\mid \, x\in X_{m+1}\right\}    
\end{equation}
and define the pullback Čech-complex as the nerve of the following pullback cover (as in Equation~\ref{eq:pullback_cover}):
$$\U_{i-1}^\varepsilon:=\{\text{Connected components of }f^{-1}_i(U)\,\mid\, U\in \U_{i}^\varepsilon\} \text{ for } i\in \{1,\dots,m+1\}.$$

\begin{theorem}[Bi-persistence for MLPs]\label{th:bi_persistence}
Let $F$ be an MLP with layers $\{\mathbb{R}^{n_{i-1}}\xrightarrow{f_i}\mathbb{R}^{n_i}\}_{i=1}^{m+1}$ where $\mathbb{R}^{n_{m+1}}=\mathbb{R}$, and let $X \subset \mathbb{R}^{n_0}$ be a dataset with $X_i = F_i(X)$ its representation at the $i$-th layer. Given a sequence of increasing scale parameters $0 \leq \varepsilon_0 < \varepsilon_1 < \cdots < \varepsilon_k$, we can yield a diagram of VR-complexes:
\begin{equation}\label{eq:diagram_simp}
\begin{tikzcd}
K_0^{\varepsilon_0} \arrow[r,"\varphi_1^{\varepsilon_0}"] \arrow[d, hookrightarrow] & K_{1}^{\varepsilon_0} \arrow[r, "\varphi_2^{\varepsilon_0}"] \arrow[d, hookrightarrow] &\cdots  \arrow[r, "\varphi_m^{\varepsilon_0}"] &K_{m}^{\varepsilon_0} \arrow[r, "\varphi_{m+1}^{\varepsilon_0}"] \arrow[d, hookrightarrow] & K_{m+1}^{\varepsilon_0}\arrow[d, hookrightarrow]\\
K_0^{\varepsilon_1} \arrow[r, "\varphi_1^{\varepsilon_1}"] \arrow[d, hookrightarrow] & K_{1}^{\varepsilon_1} \arrow[r, "\varphi_2^{\varepsilon_1}"] \arrow[d, hookrightarrow] &\cdots  \arrow[r, "\varphi_m^{\varepsilon_1}"] &K_{m}^{\varepsilon_1} \arrow[r, "\varphi_{m+1}^{\varepsilon_1}"] \arrow[d, hookrightarrow] & K_{m+1}^{\varepsilon_1}\arrow[d, hookrightarrow]\\
\vdots \arrow[d, hookrightarrow]& \vdots \arrow[d, hookrightarrow]& &\vdots\arrow[d, hookrightarrow] & \vdots\arrow[d, hookrightarrow]\\
K_0^{\varepsilon_k} \arrow[r, "\varphi_1^{\varepsilon_k}"]  & K_{1}^{\varepsilon_k} \arrow[r, "\varphi_2^{\varepsilon_k}"]  &\cdots  \arrow[r, "\varphi_m^{\varepsilon_k}"] &K_{m}^{\varepsilon_k} \arrow[r, "\varphi_{m+1}^{\varepsilon_k}"] & K_{m+1}^{\varepsilon_k}
\end{tikzcd}
\end{equation}
where $K_i^{\varepsilon_j} = N(\mathcal{U}_i^{\varepsilon_j})$, $\varphi_i^{\varepsilon_j} = N(\xi_i^{\varepsilon_j})$, and $\phi_i^j = N(\iota_i^j)$.

Applying the homology functor $H_p$ with coefficients in a field $\mathbb{F}$ then yields a commutative diagram of vector spaces and linear maps that enables two complementary analyses:

\begin{enumerate}
    \item \textbf{Layer persistence} (vertical analysis): For each fixed layer $i$, the sequence 
    $$H_p(K_i^{\varepsilon_0}) \rightarrow H_p(K_i^{\varepsilon_1}) \rightarrow \cdots \rightarrow H_p(K_i^{\varepsilon_k})$$
    forms a persistence module whose persistence diagram $\text{Dgm}_p^i$ captures the topological features of the data representation at the $i$-th layer.

    \item \textbf{MLP persistence} (horizontal analysis): For each fixed scale $\varepsilon_j$, the sequence
    $$H_p(K_0^{\varepsilon_j}) \rightarrow H_p(K_1^{\varepsilon_j}) \rightarrow \cdots \rightarrow H_p(K_{m+1}^{\varepsilon_j})$$
    forms a persistence module whose persistence diagram $\text{Dgm}_p^{\varepsilon_j}$ captures how topological features evolve through the layers of the neural network at scale $\varepsilon_j$.
\end{enumerate}
\end{theorem}
\begin{proof}
We have maps of covers $\xi_i^{\varepsilon_j}: \mathcal{U}_{i-1}^{\varepsilon_j} \to \mathcal{U}_i^{\varepsilon_j}$ which then induce simplicial maps $\varphi_i^{\varepsilon_j}: K_{i-1}^{\varepsilon_j} \to K_i^{\varepsilon_j}$. For different scales $\varepsilon_j < \varepsilon_{j+1}$ at the same layer $i$, we have natural inclusion maps. Now, for any $i \in \{1,\ldots,m+1\}$ and $j \in \{0,\ldots,k-1\}$, we have the following diagrams:
\begin{equation}
\begin{tikzcd}
\U_{i-1}^{\varepsilon_j} \arrow[r, "\xi_i^{\varepsilon_j}"] \arrow[d,hookrightarrow] & \U_i^{\xi_j} \arrow[d, hookrightarrow] \\
\U_{i-1}^{\varepsilon_{j+1}} \arrow[r, "\xi_i^{\varepsilon_{j+1}}"] & \U_i^{\varepsilon_{j+1}}
\end{tikzcd}\hspace{10mm}
\begin{tikzcd}
K_{i-1}^{\varepsilon_j} \arrow[r, "\varphi_i^{\varepsilon_j}"] \arrow[d, hookrightarrow] & K_i^{\varepsilon_j} \arrow[d, hookrightarrow] \\
K_{i-1}^{\varepsilon_{j+1}} \arrow[r, "\varphi_i^{\varepsilon_{j+1}}"] & K_i^{\varepsilon_{j+1}}
\end{tikzcd}
\end{equation}

Then, the commutativity is formally guaranteed by Proposition \ref{prop:same_mapcov}, which ensures that any two maps of covers between the same covers induce contiguous simplicial maps, and thus the diagram commutes up to contiguity, considering the map of covers obtained by composing the map of covers on each path. Since homology is invariant under contiguous maps (Fact \ref{prop:contiguous_simphom}), the diagram commutes after applying the homology functor. Applying the homology functor $H_p$ with coefficients in a field $\mathbb{F}$ to the diagram in Equation \ref{eq:diagram_simp}, we obtain a commutative diagram of vector spaces and linear maps:

\begin{equation}\label{eq:comm_hom}
\begin{tikzcd}
H_p(K_0^{\varepsilon_0}) \arrow[r] \arrow[d] & H_p(K_1^{\varepsilon_0}) \arrow[r] \arrow[d] & \cdots \arrow[r] & H_p(K_{m+1}^{\varepsilon_0}) \arrow[d] \\
H_p(K_0^{\varepsilon_1}) \arrow[r] \arrow[d] & H_p(K_1^{\varepsilon_1}) \arrow[r] \arrow[d] & \cdots \arrow[r] & H_p(K_{m+1}^{\varepsilon_1}) \arrow[d] \\
\vdots \arrow[d] & \vdots \arrow[d] & & \vdots \arrow[d] \\
H_p(K_0^{\varepsilon_k}) \arrow[r] & H_p(K_1^{\varepsilon_k}) \arrow[r] & \cdots \arrow[r] & H_p(K_{m+1}^{\varepsilon_k})
\end{tikzcd}
\end{equation}
\end{proof}
 
Layer persistence stability follows from Theorem~\ref{th:persistence_stab}, and with enough granularity and tame layer functions, we can expect the resultant persistence diagrams to be similar to the ones we could obtain by applying standard persistent homology directly on each $X_i$. However, MLP persistence involves the pullback of the open balls of the last layer, which is a difficult or impossible task in practice, as layer functions generally have no inverse functions. The general approach to approximate the pullback is applying clustering algorithms to approximate the connectivity in the pullback space, as in most Mapper implementations (for example \texttt{kmapper}\cite{KeplerMapper_JOSS}). 

\subsubsection{Combinatorial approach for MLP persistence}\label{sec:combinatorial}

In this section, we present a practical computational approach for analysing the persistent homology of MLPs in binary classification tasks. This approach avoids the need to explicitly compute pullback covers and build the Čech complexes directly on the latent representations.

\begin{definition}[Layer-wise Čech Complex]
Let $F$ be an MLP with $m$ hidden layers, $X$ a dataset, and $X_i = F_i(X)$ the image of $X$ at the $i$-th layer. Given:
\begin{itemize}
    \item A cover $\mathcal{U}_{m+1}$ of the output space $X_{m+1}$,
    \item A sequence of scale parameters $\varepsilon = \{\varepsilon_0, \varepsilon_1, \ldots, \varepsilon_m\}$,
\end{itemize}
we define the layer-wise Čech Complex $K_i^{\varepsilon_i}$ at layer $i$ as the clique complex of the subgraph of the $\varepsilon_i$-proximity graph of $X_i$ satisfying that $(u,v)$ is an edge if:
\begin{itemize}
    \item $(u,v)$ is an edge in the $1$-skeleton of $K_{i+1}^{\varepsilon_{i+1}}$, and
    \item there exists $U\in \U_{m+1}$ such that $f_{i+1} \circ \cdots \circ f_{m+1}(u)\in U$ and $f_{i+1} \circ \cdots \circ f_{m+1}(v)\in U$.
\end{itemize} 
\end{definition}

\begin{corollary}[Simplicial Tower from Čech Complexes]
The layer-wise Čech Complexes form a simplicial tower:
\begin{equation}
K_0^{\varepsilon_0} \xrightarrow{\varphi_1} K_1^{\varepsilon_1} \xrightarrow{\varphi_2} \cdots \xrightarrow{\varphi_m} K_m^{\varepsilon_m}
\end{equation}
where each $\varphi_i$ is the simplicial map induced by the layer function $f_i: X_{i-1} \rightarrow X_i$. Besides, the $n$-skeleton of $K_{i-1}^{\varepsilon_{i-1}}$ is contained in the $n$-skeleton of $K_{i}^{\varepsilon_{i}}$ for $n=\min(n_{i-1},n_i)$.
\end{corollary}

The last assertion makes it possible to compute MLP using standard persistent homology algorithms by slicing the sequence into inclusions up to a dimension. This construction allows us to track the topological evolution of data representations through the network while maintaining the relationships defined by the output layer classification. Intuitively, it can be considered an approximation of a path (not necessarily straight) from left to right in the commutative diagram of Equation~\ref{eq:comm_hom}.

\subsection{Relation to (multiscale) Mapper}\label{sec:rel_mapper}

The construction in Section~\ref{sec:pullback_cover_tower} can be understood as a generalisation of the Multiscale Mapper framework introduced by Dey et al.\cite{10.5555/2884435.2884506} and further developed in\cite{Dey_Wang_2022}. While standard Mapper applies a single filter function to analyse data, our approach applies this principle recursively through the sequential layers of a neural network.
Specifically, our methodology in Section~\ref{sec:pullback_cover_tower} can be viewed as applying multiple instances of Multiscale Mapper, where for each $i \in \{0,\dots,m\}$, we study the composite function $f_{m+1}\circ\cdots\circ f_i:X_i\rightarrow X_{m+1}$ through its pullback covers. This recursive application allows us to analyse the topological transformations induced by the entire network architecture rather than just a single filter function.
Theorem~\ref{th:stab_pullback} (Stability of Pullback Cover Towers) extends the stability results for cover perturbations developed in~\cite[Section 9.4.1]{Dey_Wang_2022}. While the original results establish stability for a single Multiscale Mapper construction, our extension handles the case of sequential pullbacks through multiple functions, reflecting the layered architecture of neural networks.

In~\cite[Section 9.4.3]{Dey_Wang_2022}, the authors introduce the intrinsic Čech filtration defined via the pullback metric. This construction measures proximity in the domain space of a continuous function, yielding a filtration that interleaves with the Multiscale Mapper. For $(c,s)$-good cover towers, they establish that Multiscale Mapper and the intrinsic Čech filtration are $2c$-interleaved, providing a computational alternative to explicitly constructing pullback covers.
Our combinatorial approach in Section~\ref{sec:combinatorial} follows a similar spirit, utilising proximity relationships in the input and latent spaces to approximate the pullback covers. However, unlike the standard intrinsic Čech filtration, our approach use a sequence of $\varepsilon$ values to guarantee connectivity in the pullback as well as inclusion up to a certain dimension in the complexes of subsequent layers, which makes MLP persistence easier to compute using standard persistent homology algorithms. Also, remark that the intrinsic Čech filtration could be applied for layer persistence but not to MLP persistence.

\subsection{Visualization of the trajectories}\label{sec:methodology_trajectories}

Given a cover tower of an MLP, it is interesting to visualise how data flows through the covers related by a map of covers. 

\begin{definition}[Trajectories of a cover tower]
Given a dataset $X \subset \mathbb{R}^{n_0}$, an MLP $F$ and a cover tower $\{\U_{i-1}\rightarrow \U_{i}\}_{i=0}^{m+1} $ of $F$. The trajectory of a point $x \in X$ is defined as
$$T(x) = (j_0, j_1, ..., j_{m+1})$$
where each $j_i \in J_i$ is the index $j_i=\xi_i(j_{i-1})$ for $i=m+1,\dots 1$ such that $U_{j_0} \in \U_0$ with $x \in U_{j_0}$.    
\end{definition}

\begin{proposition}\label{prop:graphintersection}
Let $x, y \in X$ be two points with trajectories induced by a sequence of cover maps $T(x) = (j_0^x, j_1^x, ..., j_{m+1}^x)$ and $T(y) = (j_0^y, j_1^y, ..., j_{m+1}^y)$ from a cover tower $\{\U_{i-1}\rightarrow \U_{i}\}_{i=0}^{m+1} $. If $j_i^x = j_i^y$ for some layer $i$, then $j_k^x = j_k^y$ for all $k > i$.
\end{proposition}
\begin{proof}
This follows directly from the construction of the map of covers. By the definition, all points in $U_j$ map to the same open set in the next layer, determined by the map of covers $\xi_i$. By induction, the same applies to all subsequent layers.
\end{proof}

The analysis of trajectories offers a powerful lens for understanding how MLPs organise and separate data points across layers. By analysing trajectories across the entire network, we can construct a global view of how the MLP organises data.

\begin{definition}[Trajectory Graph]
Given a cover tower $\{\U_{i-1}\xrightarrow{\xi_i} \U_i\}_{i=0}^{m+1}$ with indices sets $J_i$.The trajectory graph $G_T = (V, E)$ is a directed graph where the set of vertices are $\bigcup_{i=0}^{m+1}J_i$ and the edges are given by $\xi_i$.
\end{definition}

The trajectory graph provides a comprehensive view of how data flows through the network. By colouring vertices according to the predominant class of points they contain, we can visualise how the network progressively organises data by class across layers.
\section{Experiments}\label{sec:experiments}

In this section, we validate our methodology on two distinct datasets: an illustrative 2-dimensional toy example and a high-dimensional real-world dataset of measurements of fetal heart rate and uterine contraction features on cardiotocograms classified by expert obstetricians~\cite{cardiotocography_193}. In practice, we use Vietoris-Rips complexes instead of Čech complexes by checking pairwise distances and adding higher-dimensional simplices when possible. This simplification does not alter the theoretical results while providing computational advantages.

All MLPs were trained using the Adam optimiser with default parameters. Training proceeded for 1000 epochs for the synthetic example and 5000 epochs for the cardiotocography dataset.

\subsection{Illustrative example}
\label{subsec:example}

We demonstrate our approach using a canonical non-linearly separable classification problem consisting of two concentric circles in $\mathbb{R}^2$ (Figure~\ref{fig:input})\footnote{This example is inspired by Colah's blog: \\\url{https://colah.github.io/posts/2014-03-NN-Manifolds-Topology/}}. This dataset presents a clear topological challenge: the two classes form concentric 1-cycles with a natural embedding in 2D space, but cannot be separated by any linear decision boundary. 

We employed an MLP $F:\mathbb{R}^2\rightarrow \mathbb{R}^3\rightarrow \mathbb{R}$ with sigmoid activation functions on all layers. This architecture, while simple, requires the network to learn a non-trivial transformation from the input space to a linearly separable representation. The hidden layer provides just enough capacity to solve the problem, with an intermediate 3D representation that enables linear separability. Once $F$ is trained to 100\% accuracy, we compute the latent representations $X_1$ and $X_2$, which correspond to the outputs of the hidden and final layers, respectively (see Figures~\ref{fig:dataset_classified} and \ref{fig:latent_classified}). 

\begin{figure}[h]
    \centering
    \begin{subfigure}[b]{0.3\linewidth}
        \includegraphics[width=\linewidth]{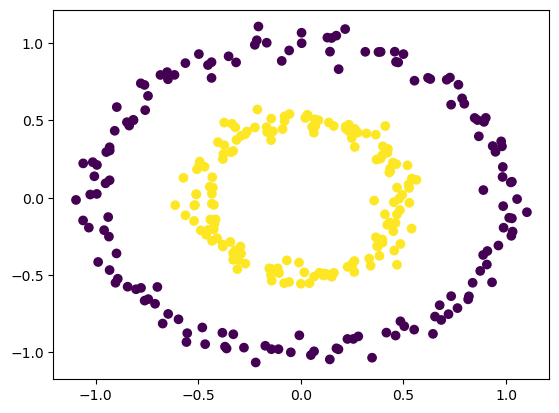}
        \caption{Dataset for binary classification composed of two concentric cycles.}
        \label{fig:input}
    \end{subfigure}
    \hfill
    \begin{subfigure}[b]{0.3\linewidth}
        \includegraphics[width=\linewidth]{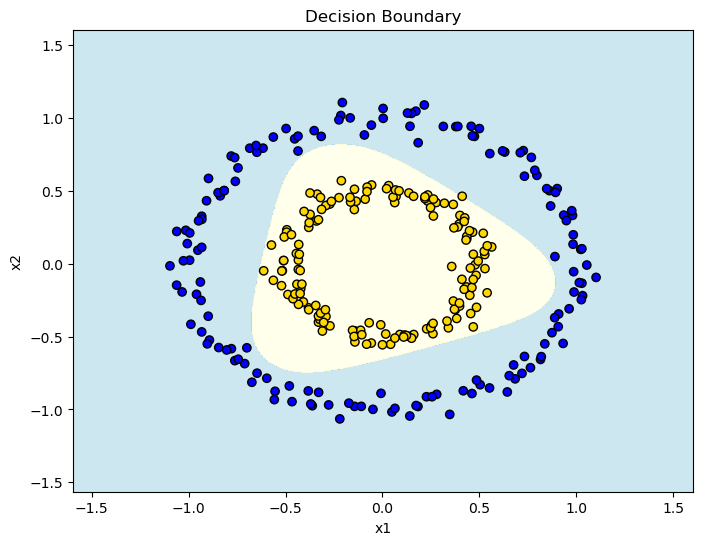}
        \caption{$X_0$ input dataset with decision boundary of the MLP $F$.}
        \label{fig:dataset_classified}
    \end{subfigure}
    \hfill
    \begin{subfigure}[b]{0.3\linewidth}
        \includegraphics[width=\linewidth]{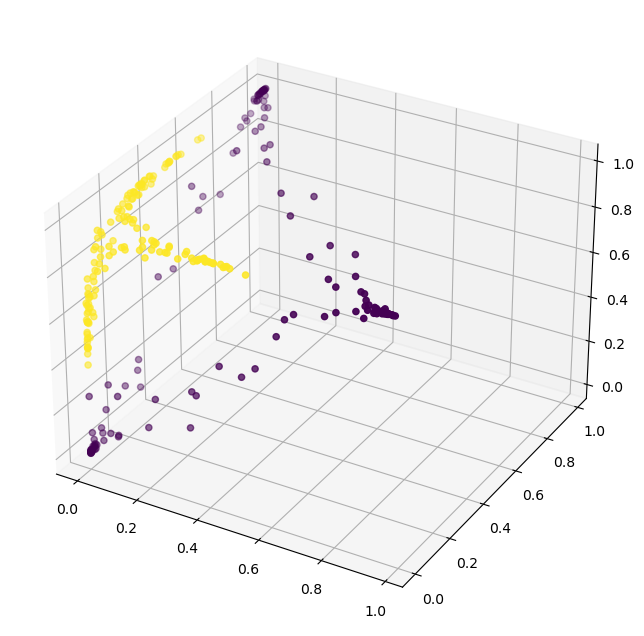}
        \caption{$X_1$ latent representation colored based on the MLP $F$ output.}
        \label{fig:latent_classified}
    \end{subfigure}
    \caption{(Section~\ref{subsec:example}) Classification problem and latent representation}
    \label{fig:Ex_setup}
\end{figure}

We now analyse how the classification is performed and extract information about the internal operations. The first step involves computing the layer persistence diagrams for $X_0$, $X_1$ and $X_2$ (see Figure~\ref{fig:ex1_layer_persistence}). The layer persistence diagrams provide insight into the topological evolution of the dataset through $F$. In Figure~\ref{fig:ex1_PH_L0}, we observe that for $X_0$ there are two prominent persistent 1-cycles alive during intervals $[0.12,0.72]$ and $[0.17,0.45]$, which correspond to the two concentric cycles. In Figure~\ref{fig:ex1_PH_L1}, we also observe two prominent 1-cycles for $X_1$ but with shorter lifespans, specifically during intervals $[0.18,0.43]$ and $[0.2,0.42]$. Finally, for $X_2$ we can only study connected components, but we observe that at scale $0.2$, there are two connected components that eventually merge at scale $0.21$. Consequently, we choose as the cover $\mathcal{U}_2$ for $X_2$ those two connected components.

\begin{figure}[h]
    \centering
    \begin{subfigure}[b]{0.3\linewidth}
        \includegraphics[width=\linewidth]{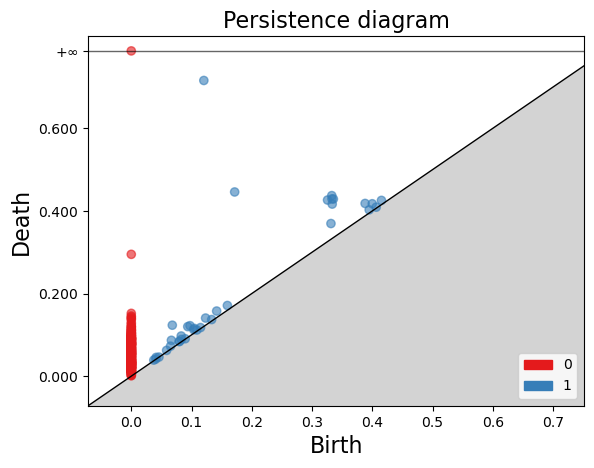}
        \caption{Layer persistence of $X_0$.}
        \label{fig:ex1_PH_L0}
    \end{subfigure}
    \hfill
    \begin{subfigure}[b]{0.3\linewidth}
        \includegraphics[width=\linewidth]{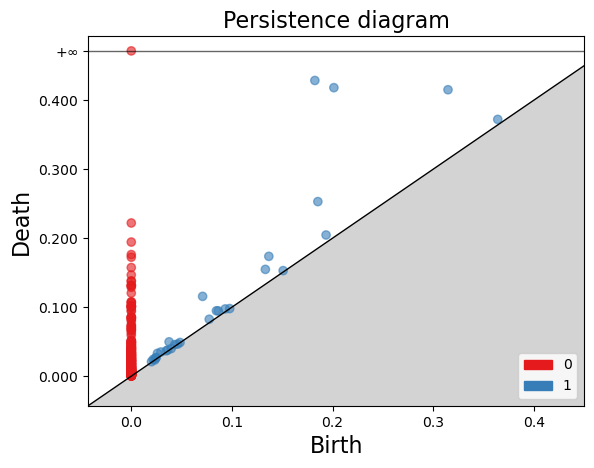}
        \caption{Layer persistence of $X_1$.}
        \label{fig:ex1_PH_L1}
    \end{subfigure}
    \hfill
    \begin{subfigure}[b]{0.3\linewidth}
        \includegraphics[width=\linewidth]{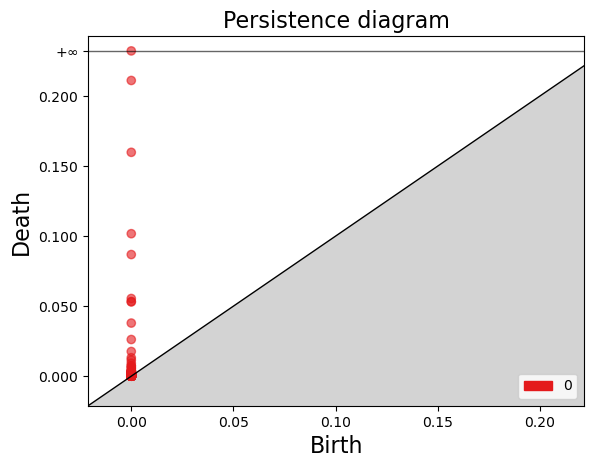}
        \caption{Layer persistence of $X_2$.}
        \label{fig:ex1_PH_L2}
    \end{subfigure}
    \caption{(Section~\ref{subsec:example}) Layer persistence for the two concentric cycles classification problem.}
    \label{fig:ex1_layer_persistence}
\end{figure}

Next, we compute MLP persistence (following Section~\ref{sec:combinatorial}) to track the relationships between homology components throughout the MLP. By examining these diagrams, we identify suitable scale parameters for the MLP persistence analysis. Specifically, we choose $\varepsilon = \{0.5, 0.4, 0.2\}$ for the input, hidden, and output layers, respectively. These values are selected by considering the scales at which our homology components of interest remain active. We then construct the layer-wise Vietoris-Rips complexes for each layer (see Figure~\ref{fig:ex1_MLP_VR}) and compute MLP persistence for these parameters. Taking into account the dimensions of $X_0$, $X_1$ and $X_2$, we obtain the following persistence sequences:
\begin{equation*}
\begin{split}
&H_0(K_0^{\varepsilon_0})\rightarrow H_0(K_1^{\varepsilon_1})\rightarrow H_0(K_2^{\varepsilon_2})\\
&H_1(K_0^{\varepsilon_0})\rightarrow H_1(K_1^{\varepsilon_1})\rightarrow 0
\end{split}    
\end{equation*}

\begin{figure}[h]
    \centering
    \begin{subfigure}[b]{0.35\linewidth}
        \includegraphics[width=\linewidth]{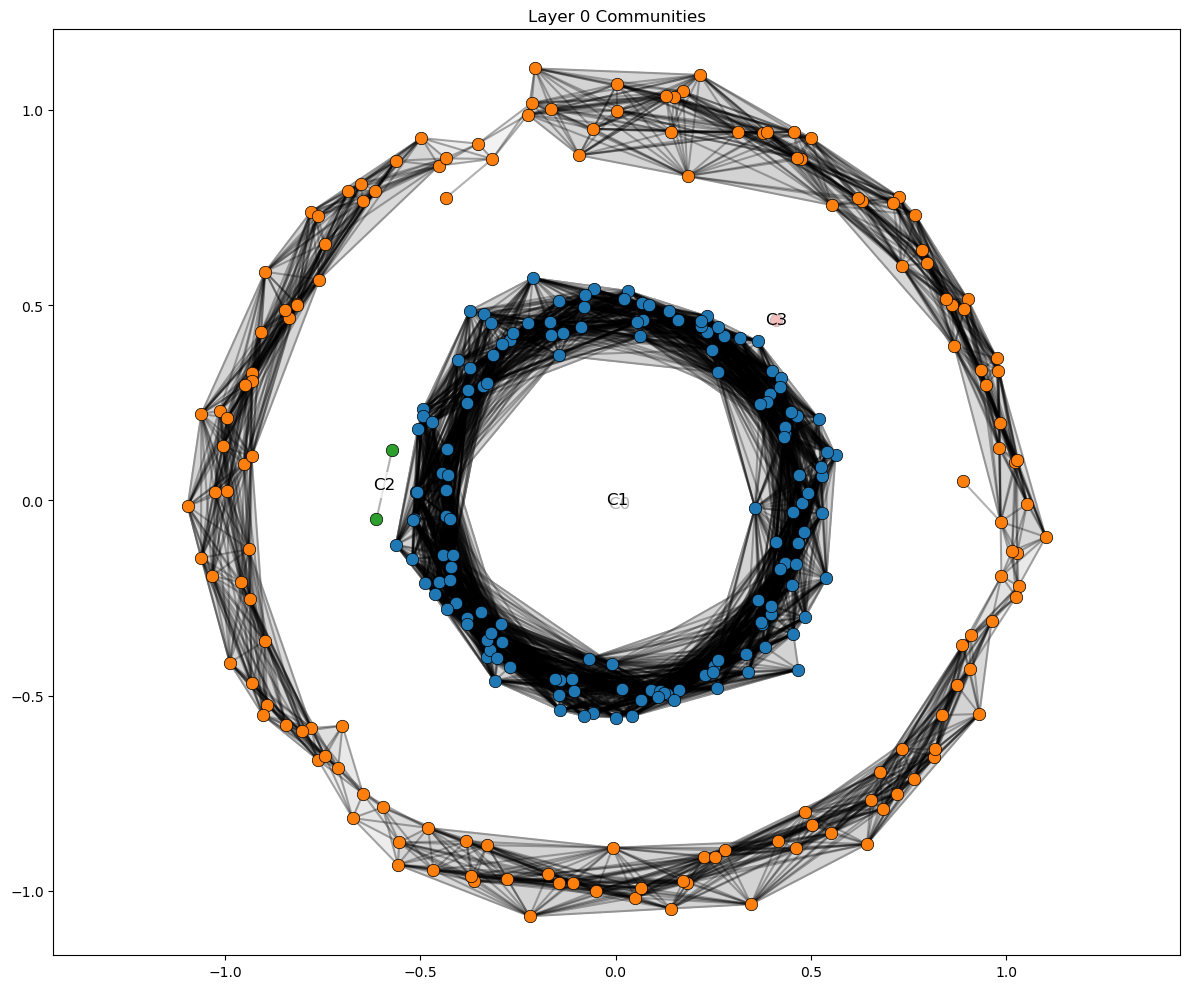}
        \caption{Layer-wise VR 1-skeleton for $X_0$ with $\varepsilon_0=0.5$.}
        \label{fig:ex1_VR0}
    \end{subfigure}
    \hspace{10mm}
    \begin{subfigure}[b]{0.35\linewidth}
        \includegraphics[width=\linewidth]{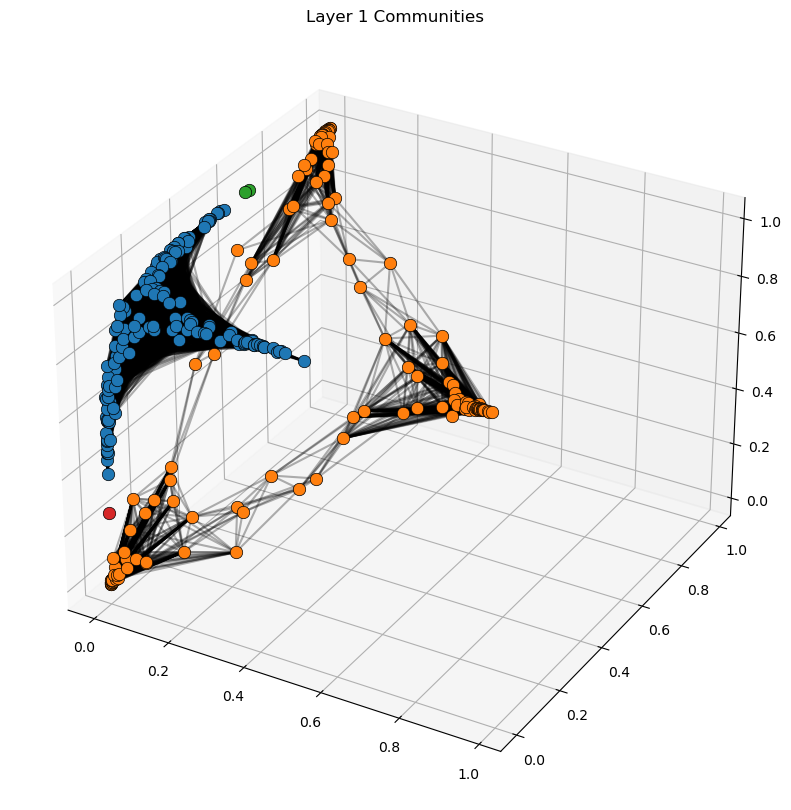}
        \caption{Layer-wise VR 1-skeleton for $X_1$ with $\varepsilon_1=0.4$.}
        \label{fig:ex1_VR1}
    \end{subfigure}
\caption{(Section~\ref{subsec:example}) Layer-wise 1-skeleton for MLP persistence for the two concentric cycles classification problem. Each connected component has a different color.}
    \label{fig:ex1_MLP_VR}
\end{figure}

In Figure~\ref{fig:Ex1_MLP_pers}, we observe that at layer 0 there are four connected components, but two of them disappear at the output layer, as expected, since the MLP correctly classifies the two classes. For the 1-cycles, we have two at layer 0, which correspond to the two cycles of the input dataset. However, one of them disappears at layer 1 (hidden layer) while a new one emerges. While layer persistence might suggest that the two 1-cycles persist unchanged through layers 0 and 1, MLP persistence reveals that only one of them remains active in both layers.

Finally, we analyse the trajectories of the connected components of the layer-wise Vietoris-Rips complexes in Figure~\ref{fig:ex1_trajectories}. We observe that the trajectories do not overlap throughout the entire MLP. For one class, we have one connected component that persists across all latent representations, while for the other class, we have one main component and two smaller ones. Trajectories provide a better understanding of the classification process and the evolution of connected components. However, since this toy example achieves perfect classification, all nodes in the graph are pure (no mixing between classes). The two dominant trajectories are $(1,1,1)$ with 150 points and $(0,0,0)$ with 147 points. The remaining two trajectories $(2,2,0)$ and $(3,3,0)$ contain only 2 and 1 points, respectively.

\begin{figure}[h]
    \centering
    \begin{subfigure}[t]{0.35\linewidth}
        \centering
        \begin{minipage}[t]{\linewidth}
            \includegraphics[width=\linewidth]{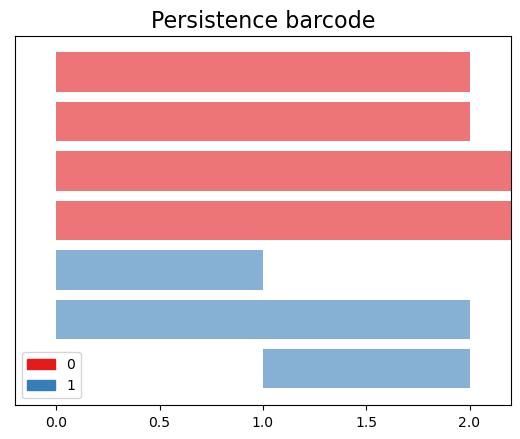}
            \caption{Persistence barcode of the MLP for the two concentric cycles classification problem with $\varepsilon=\{0.5,0.4,0.2\}$. The $x$-axis represents the layers 0 (input), 1 (hidden), and 2 (output). Bars are colored red for connected components and blue for 1-cycles.}
            \label{fig:Ex1_MLP_pers}
        \end{minipage}
    \end{subfigure}
    \hspace{10mm}
    \begin{subfigure}[t]{0.45\linewidth}
        \centering
        \begin{minipage}[t]{\linewidth}
            \includegraphics[width=\linewidth]{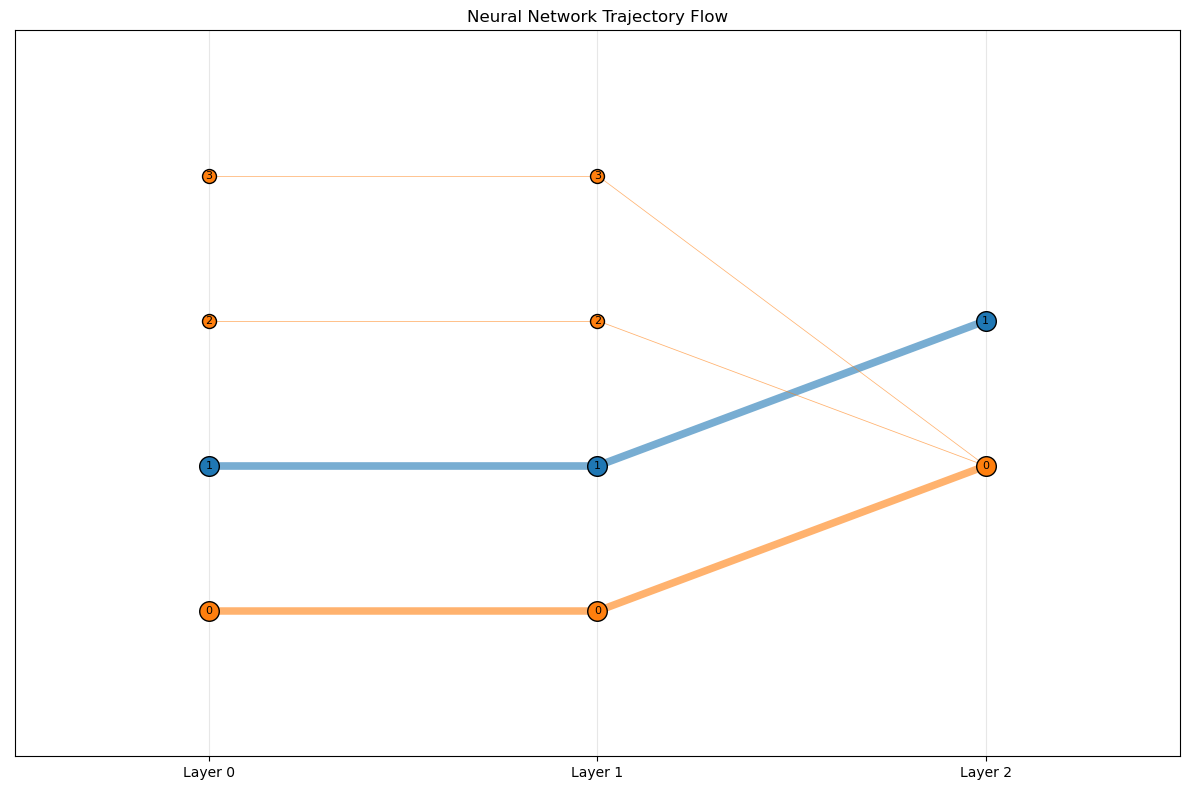}
            \caption{Graph of the trajectories of the connected components of the layer-wise Vietoris-Rips complexes with $\varepsilon=\{0.5,0.4,0.2\}$, colored by class.}
            \label{fig:ex1_trajectories}
        \end{minipage}
    \end{subfigure}
    \caption{(Section~\ref{subsec:example}) Topological representations of the two concentric cycles classification problem.}
    \label{fig:ex1_combined}
\end{figure}

\subsection{Cardiotocography classification}\label{sec:cardiotocography}

We analyse the cardiotocography dataset~\cite{cardiotocography_193}, which comprises 2,126 fetal cardiotocograms with 21 features and considers its classification problem version with 10 NST classes. NST classes assist obstetricians in assessing fetal health and identifying potential risks. We transform this dataset into a binary classification problem by grouping the NST classes into normal (Classes 1-4) versus concerning patterns (Classes 7-10).

In this experiment, to reduce the number of simplices in the Vietoris-Rips complex when computing layer persistence, we sparsify the dataset at each latent representation such that the squared distance between any two points is greater than or equal to 0.05. This substantially reduces the complex size and enables faster computation while not changing the persistence diagram by more than 0.05. For MLP persistence computation, we also sparsify the data, but only at the input layer, while maintaining the same data for the remaining layers.

As an initial exploration step, we consider an MLP $F:\mathbb{R}^{21}\rightarrow\mathbb{R}^{32}\rightarrow \mathbb{R}^{16}\rightarrow \mathbb{R}^{8}\rightarrow \mathbb{R}^{4}\rightarrow\mathbb{R}$ and compute MLP persistence along with the trajectories. By analysing the trajectories, we discover that all homological changes occur within the first three layers and then remain stable. Consequently, we infer that the model is overparameterized and can be reduced in size. Based on this insight, we train a new MLP with fewer layers $G:\mathbb{R}^{21}\rightarrow\mathbb{R}^{32}\rightarrow\mathbb{R}$, achieving 0.96 accuracy. 

In Figure~\ref{fig:ex2_layer_persistence_2}, we observe the persistence diagrams for layer persistence. We notice that from layer 0 to layer 1, the number of 1-cycles is reduced. We then compute MLP persistence using $\varepsilon=\{1,2.5,0.2\}$ and obtain the persistence barcode shown in Figure~\ref{fig:Ex2_MLP_pers_2}. There, we observe that already at layer 1, most connected components disappear, and the data clusters into three connected components. For the 1-cycles, a similar pattern occurs: only two of the cycles remain across both layers, while two new ones appear at the hidden layer. The evolution of the connected components can be more intuitively appreciated in the trajectory graph in Figure~\ref{fig:ex2_trajectories_2}, where we observe that from the beginning, there are two larger clusters that persist throughout the entire MLP. We can analyse the two dominant trajectories $(0,0,0)$ and $(1,1,1)$, both of which have a purity of almost 1, indicating that the clusters clearly differentiate between both classes.

\begin{figure}[h]
    \centering
    \begin{subfigure}[b]{0.3\linewidth}
        \includegraphics[width=\linewidth]{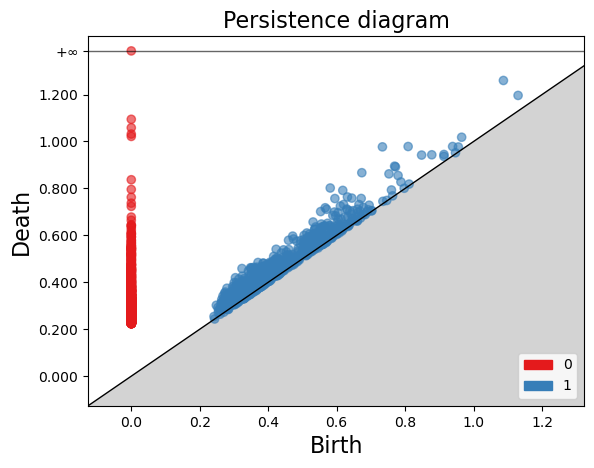}
        \caption{Layer persistence of $X_0$.}
        \label{fig:ex2_PH_L0_2}
    \end{subfigure}
    \hfill
    \begin{subfigure}[b]{0.3\linewidth}
        \includegraphics[width=\linewidth]{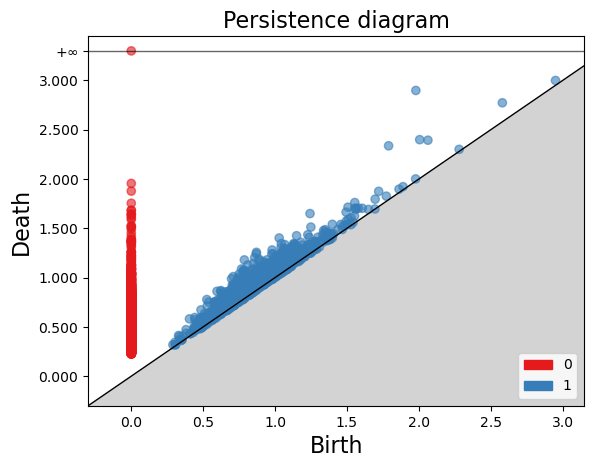}
        \caption{Layer persistence of $X_1$.}
        \label{fig:ex2_PH_L1_2}
    \end{subfigure}
    \hfill
    \begin{subfigure}[b]{0.3\linewidth}
        \includegraphics[width=\linewidth]{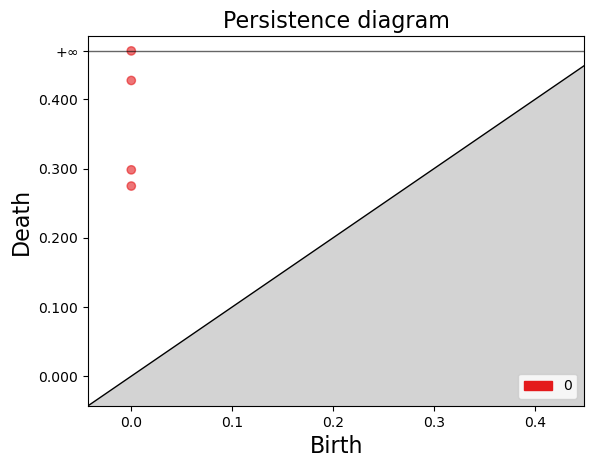}
        \caption{Layer persistence of $X_2$.}
        \label{fig:ex2_PH_L2_2}
    \end{subfigure}
    \caption{(Section~\ref{sec:cardiotocography}) Layer persistence for the cardiotocography classification problem using model $G$.}
    \label{fig:ex2_layer_persistence_2}
\end{figure}

\begin{figure}[h]
    \centering
    \begin{subfigure}[t]{0.35\linewidth}
        \centering
        \begin{minipage}[t]{\linewidth}
            \includegraphics[width=\linewidth]{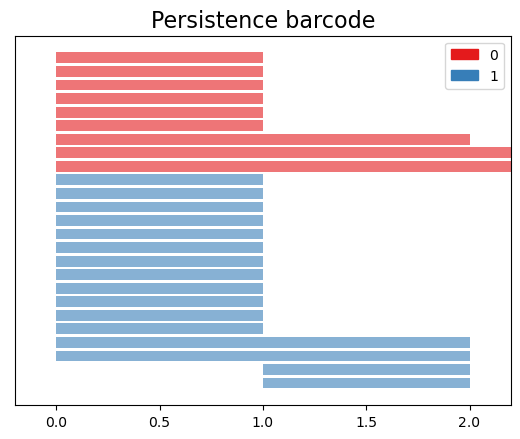}
            \caption{Persistence barcode of the MLP for the cardiotocography classification problem with $\varepsilon=\{1,2.5,0.2\}$. The $x$-axis represents the layer indices.}
            \label{fig:Ex2_MLP_pers_2}
        \end{minipage}
    \end{subfigure}
    \hspace{10mm}
    \begin{subfigure}[t]{0.45\linewidth}
        \centering
        \begin{minipage}[t]{\linewidth}
            \includegraphics[width=\linewidth]{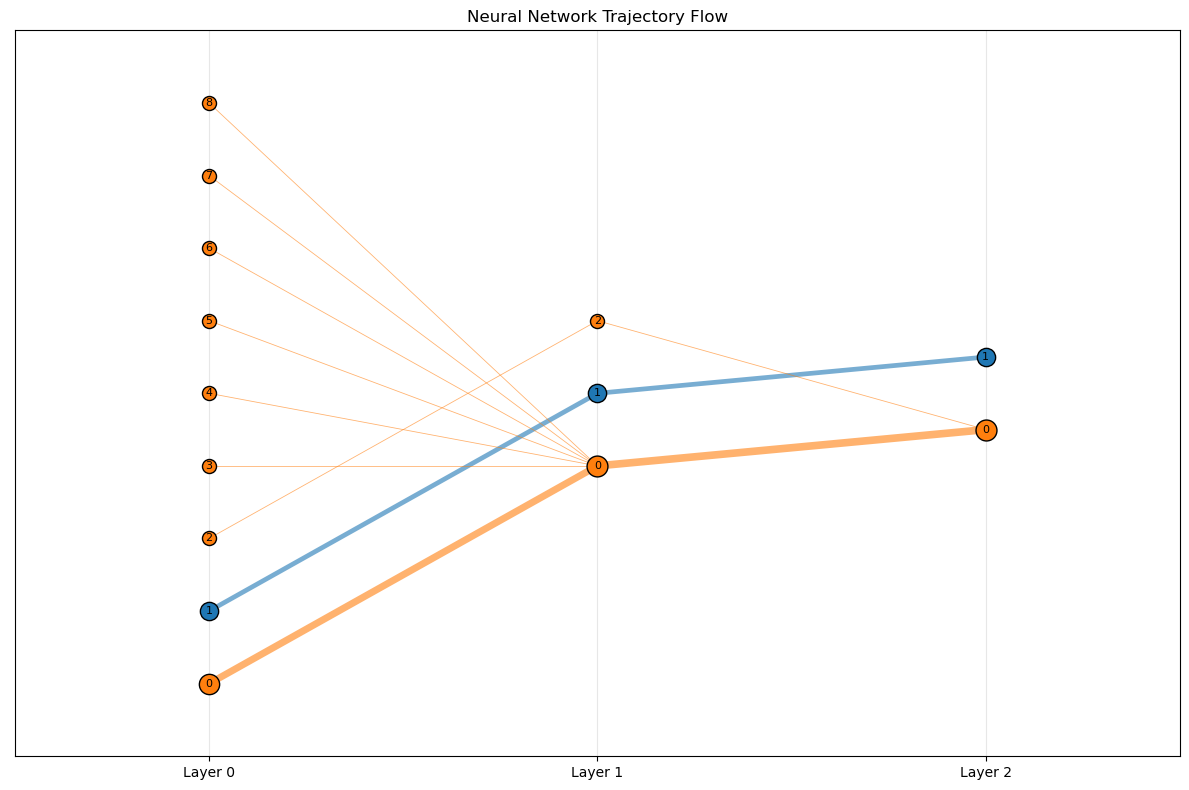}
            \caption{Graph of the trajectories of the connected components of the layer-wise Vietoris-Rips complexes with $\varepsilon=\{1,2.5,0.2\}$, colored by the predominant class.}
            \label{fig:ex2_trajectories_2}
        \end{minipage}
    \end{subfigure}
    \caption{(Section~\ref{sec:cardiotocography}) MLP persistence and trajectories for model $G$ in the cardiotocography classification problem.}
    \label{fig:ex2_combined_2}
\end{figure}


\section{Conclusions and future work}\label{sec:conclusions}

In this paper, we have introduced a novel topological framework for analysing and interpreting the internal representations of Multilayer Perceptrons. Our approach leverages concepts from computational topology to provide insights into how neural networks progressively transform and organise data across layers. By developing a bi-persistence framework inspired by Multiscale Mapper, we enable the systematic tracking of topological features through both layer-wise evolution and across the entire network architecture.

The key contributions of our work include:

\begin{enumerate}
    \item A formal methodology for constructing simplicial towers that capture the topological evolution of data representations in MLPs, connected by simplicial maps that reflect the network's transformations.
    \item A bi-persistence framework that enables two complementary analyses: layer persistence, which captures the intrinsic topological structure at each layer, and MLP persistence, which tracks how features evolve across the network.
    \item Stability results that provide theoretical guarantees for our topological descriptors under cover perturbations, ensuring the robustness of our approach.
    \item A trajectory-based visualisation tool that transforms abstract topological concepts into interpretable representations of how data points flow through the network.
    \item A practical combinatorial approach for computing MLP persistence that makes the framework accessible and computationally feasible.
\end{enumerate}
Despite the promising results, our framework has several limitations that open avenues for future research. While we currently analyse trained networks as static entities, extending our approach to track topological evolution throughout the training process could yield valuable insights into learning dynamics and convergence properties. Although our methodology provides interpretable representations of network behaviour, there remains significant potential to develop more user-friendly tools that translate these topological insights into actionable knowledge for practitioners without expertise in algebraic topology. Additionally, the computational complexity of constructing simplicial towers and computing exact persistent homology presents challenges for analysing very large networks or datasets. Therefore, a key direction for our future work is developing efficient approximation techniques that preserve essential topological insights while substantially reducing computational demands, making our approach more practical for real-world applications.

\subsubsection*{Code availability}
The code is available in the following GitHub repository:\\
\url{https://github.com/EduPH/Latent-Space-Topology-Evolution-in-Multilayer-Perceptrons}. Persistent homology computations were carried out using the Python library \texttt{Gudhi}~\cite{gudhi:urm}.

\subsubsection*{Aknowledgements}
This project has received funding from the European Union's Horizon 2020 research and innovation programme under the Marie Skłodowska-Curie grant agreement No 101153039 (CHALKS). I would like to thank Professor Vidit Nanda and Dr. Chunyin Siu for their insightful conversations and valuable comments.

\bibliography{references}{}
\bibliographystyle{plain}
\end{document}